\documentclass{article}

\PassOptionsToPackage{numbers,compress}{natbib}
\usepackage[preprint]{neurips_2026}

\usepackage[utf8]{inputenc} % allow utf-8 input
\usepackage[T1]{fontenc}    % use 8-bit T1 fonts
\usepackage{hyperref}       % hyperlinks
\usepackage{url}            % simple URL typesetting
\usepackage{booktabs}       % professional-quality tables
\usepackage{amsfonts}       % blackboard math symbols
\usepackage{nicefrac}       % compact symbols for 1/2, etc.
\usepackage{microtype}      % microtypography
\usepackage{xcolor}         % colors
\usepackage{amsmath,amsfonts}
\usepackage{algorithmic}
\usepackage{algorithm}
\usepackage{array}
\usepackage{textcomp}
\usepackage{stfloats}
\usepackage{verbatim}
\usepackage{cite}
\usepackage{amsthm}
\usepackage{multirow}
\usepackage{threeparttable}
\usepackage[most]{tcolorbox}
\usepackage{graphicx}
\usepackage{subcaption}
\usepackage{amssymb}

\newtheorem{proposition}{Proposition}

\title{MLGIB: Multi-Label Graph Information Bottleneck for Expressive and Robust Message Passing}

\author{
  Chaokai Wu, Haofu Shi, Ningxuan Ma, Jianghong Ma, Xiaofeng Zhang 
}

\begin{document}

\maketitle

\begin{abstract}
%Over-squashing refers to the information loss when exponentially growing structural information is compressed into fixed-dimensional representations. While this phenomenon has been widely studied as a structural bottleneck, we observe that its effects are not fully explained in multi-label graphs, where nodes are associated with partially overlapping and inconsistent label signals. In this setting, information degradation is also driven by difficulties in discriminating informative signals under different labels. To address this limitation, we reinterpret message passing as a constrained information transmission problem from an information-theoretic perspective. Based on this formulation, we propose the Multi-Label Graph Information Bottleneck (MLGIB), a principled framework that balances predictive sufficiency and representation compactness under label-induced noisy information. We derive tractable variational bounds over a Markovian dependence space, enabling principled control over redundant information while preserving task-relevant signals. Finally, we instantiate MLGIB into a differentiable message-passing architecture for general multi-label graphs. Extensive experiments on multiple benchmarks demonstrate consistent improvements over existing methods, validating the effectiveness and generality of the proposed framework.

Graph Neural Networks (GNNs) suffer from over-squashing in deep message passing, where information from exponentially growing neighborhoods is compressed into fixed-dimensional representations. We show that this issue becomes a distinct failure mode in multi-label graphs: neighboring nodes often share only limited labels while differing across many irrelevant ones, causing predictive signals to be diluted by noisy label information. To address this challenge, we propose the \textit{Multi-Label Graph Information Bottleneck} (MLGIB), which formulates multi-label message passing as constrained information transmission under irrelevant label noise. MLGIB balances expressiveness and robustness by preserving predictive label signals while suppressing irrelevant noise. Specifically, it constructs a Markovian dependence space and derives tractable variational bounds, where the lower bound maximizes mutual information with target labels and the upper bound constrains redundant source information. These bounds lead to an end-to-end label-aware message-passing architecture. Extensive experiments on multiple benchmarks demonstrate consistent improvements over existing methods, validating the effectiveness and generality of the proposed framework.
\end{abstract}
\section{Introduction}

Graph Neural Networks (GNNs) are widely used for modeling structured data \citep{DBLP:conf/aaai/AziziKHB26,DBLP:conf/aaai/HevapathigeWZ26,DBLP:conf/aaai/LiuY26,DBLP:conf/iclr/BergnaCOLH25,DBLP:conf/iclr/LiG0025}, but deeper message passing often suffers from over-squashing, where information from exponentially growing neighborhoods is compressed into fixed-dimensional embeddings, causing the loss of long-range dependencies \citep{DBLP:conf/iclr/0002Y21,DBLP:conf/iclr/ToppingGC0B22}. While this issue has been studied mainly in homophilous graphs \citep{DBLP:conf/nips/JamadandiRB24,DBLP:conf/nips/ZhuYZHAK20}, we show that it induces a distinct failure mode in multi-label graphs.

Unlike conventional homophilous graphs \citep{wo2025local,loveland2025unveiling,jiang2026does} under the single-label setting, where neighboring nodes are typically assumed to share the same class label, multi-label graphs \citep{yu2014large,bhatia2015sparse,bei2025correlation, wu2025graph,sun2025multi} present a more intricate propagation scenario. As illustrated in Fig. \ref{fig:intro}(a), nodes in multi-label graphs typically share only a few labels while differing across many irrelevant ones. Additional statistics in Appendix Fig. \ref{fig:statistic} further support this observation. Thus, message passing \citep{chen2025beyond,walkega2025expressive} expands the receptive field but also introduces increasingly noisy, partially relevant label signals, making aggregation ambiguous and effectively turning message passing into a noisy channel \citep{DBLP:conf/iclr/0002Y21,DBLP:conf/icml/LiangBSXPS25,DBLP:conf/cikm/SaberS25}. As shown in Fig. \ref{fig:intro}(b), after $l$ layers of message passing, a target node $v$ is inundated with mixed signals from its multi-hop neighborhood, where a small fraction of label information may be relevant while a large portion corresponds to irrelevant labels. Such noisy accumulation can quickly saturate node representations before meaningful long-range dependencies are effectively integrated. This indicates that over-squashing is not merely caused by structural compression, but also by the limited ability of GNNs to extract \textit{informative signals from partially overlapping label information}.

This observation exposes two limitations of existing methods, as shown in Fig. \ref{fig:intro}(c). 
(1) Existing multi-label graph learning methods \citep{zhou2021multi,xiao2022semantic,gao2019semi} model label correlations, but \textit{lack sufficient expressiveness to preserve informative label signals} 
when they are entangled with irrelevant ones during multi-label propagation,
%under noisy multi-label propagation, 
causing predictive information to be diluted by irrelevant labels. 
(2) Current over-squashing mitigation methods, such as graph rewiring \citep{karhadkar2023fosr,nguyen2023revisiting,DBLP:conf/iclr/ToppingGC0B22}, improve structural connectivity but \textit{lack robustness to filter irrelevant messages}, thereby amplifying noise in multi-label graphs. 
Together, these limitations suggest that effective multi-label message passing requires both expressiveness to preserve predictive label signals and robustness to suppress irrelevant noise.

To address the above limitations, we formulate multi-label message passing as an information-theoretic problem of transmitting predictive label signals under irrelevant noise. As shown in Fig. \ref{fig:intro}(d), we propose the Multi-Label Graph Information Bottleneck (MLGIB), which balances two complementary objectives: \textit{expressiveness}, which preserves predictive label information during propagation, and \textit{robustness}, which suppresses noise from irrelevant labels. To realize this principle, MLGIB constructs a Markovian dependence space tailored to multi-label graphs and optimizes it via tractable variational bounds. The upper bound limits redundant information induced by irrelevant labels, while the lower bound maximizes the mutual information of relevant predictive label signals. These bounds guide the design of an end-to-end differentiable message-passing architecture.

 \vspace{0mm}

\begin{itemize}
    \item To the best of our knowledge, this is the \textit{first} work to extend the Information Bottleneck principle to multi-label graph learning, revealing that partial label overlap exacerbates over-squashing during message passing and motivates principled information filtering.

    \item We propose the Multi-Label Graph Information Bottleneck (MLGIB), a principled framework that formulates multi-label message passing as a constrained information transmission problem, enhancing the expressiveness and robustness in message passing.
    %By introducing a Markovian dependence space, we derive tractable variational bounds that balance predictive sufficiency and representation compactness while explicitly suppressing irrelevant label interference.

    \item We instantiate MLGIB into a practical message-passing architecture tailored to multi-label graphs, incorporating label-aware information filtering into the propagation process.

    \item We conduct extensive experiments on multiple benchmark datasets, demonstrating consistent improvements over existing methods across different evaluation settings.
\end{itemize}
\begin{figure}
    \centering
    \includegraphics[width=1\linewidth]{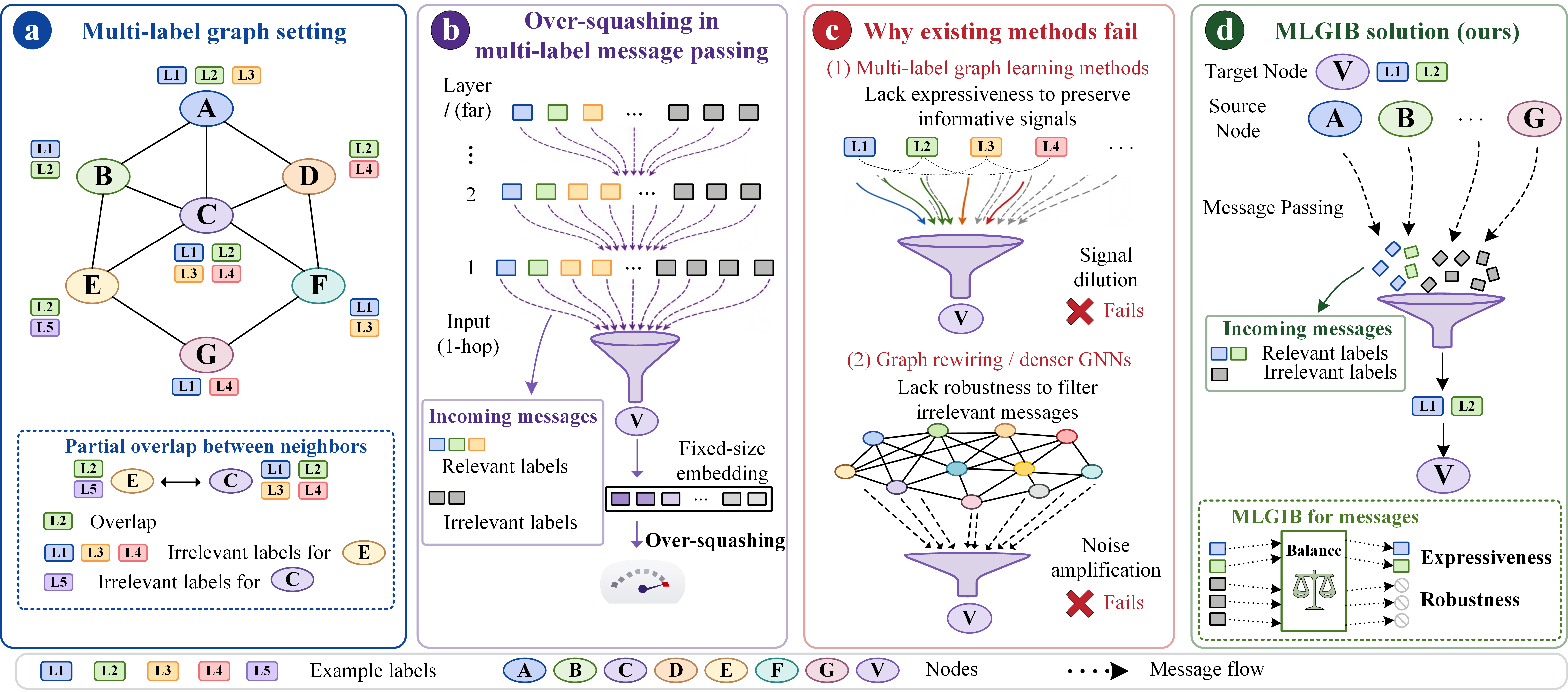}
    \caption{(a) The partial overlap problem of labels in multi‑label graphs; (b) The manifestation of over‑squashing in multi‑label graphs, where exponentially growing label information is compressed into a fixed‑dimensional vector; (c) Existing methods either struggle to extract truly informative label signals or introduce additional noise; (d) MLGIB can extract useful label information while reducing interference from irrelevant noise, exhibiting excellent expressiveness and robustness.}
    \label{fig:intro} \vspace{-4mm}
\end{figure}
\section{Preliminaries and Notations}
% Consider a multi-label graph denoted by $G=(V,E,X)$, where $V$ represents the set of nodes and $E$ represents the set of edges. Let $A$ denote the adjacency matrix of $G$, where $A_{ij}=1$ if $(i,j) \in E$, and let $D$ be the degree matrix. 

\subsection{Information Bottleneck in Deep Networks}
%Given the input data $X$ and intermediate encoding $Z_X$, the goal of deep networks is to learn an encoding that is maximally informative about target $Y$. $Z_X$ is defined by a parametric encoder $p(z|x)$, where $x$ and $z$ are instances of $X$ and $Z_X$. This is measured by the mutual information between encoding and target:
Given input data $X$ and target variable $Y$, deep networks aim to learn an intermediate representation $Z_X$ that is maximally informative for predicting $Y$. 
$Z_X$ is generated by a parametric encoder $p(z|x)$, where $x$ and $z$ denote instances of $X$ and $Z_X$, respectively. 
The predictive information in $Z_X$ can be measured by the mutual information between $Z_X$ and $Y$:
\begin{equation}
\resizebox{0.4\linewidth}{!}{$        \displaystyle
    I(Z_X;Y)=\int \mathrm{d} z \, \mathrm{d} y \, p(z,y) \log \frac{p(z,y)}{p(z)p(y)},
    $}
\end{equation}
where $y$ is an instance of $Y$. %If the sole objective is to maximize mutual information, setting $Z_X=X$ would suffice. However, such an approach does not yield a meaningful representation. The Information Bottleneck (IB) \citep{tishby2000information,tishby2015deep,saxe2019information} principle addresses this limitation by imposing $I(Z_X;X)\le I_c$ on the mutual information between $Z_X$ and $X$, thereby refining $Z_X$ through purification, where $I_c$ is the information constraint. This objective function is detailed as follows:
However, directly maximizing $I(Z_X;Y)$ may lead to a trivial solution, e.g., setting $Z_X=X$, which preserves all input information but does not yield a compact representation. 
The Information Bottleneck (IB) principle \citep{tishby2000information,tishby2015deep,saxe2019information} addresses this by constraining the information $Z_X$ retains about the input $X$. 
Specifically, it imposes the constraint $I(Z_X;X)\le I_c$, where $I_c$ denotes the information budget, leading to the constrained optimization problem:
\begin{equation}
    \max I(Z_X;Y)\; s.t.\, I(X;Z_X)\le I_c.
\end{equation}
%And then we can optimize the objective function by a Lagrange multiplier $\beta$,
By introducing a Lagrange multiplier $\beta$, the objective can be equivalently written as
\begin{equation}
    \min\, \text{IB}(Z_X,Y,X)\triangleq -I(Z_X;Y)+\beta I(X;Z_X).
\end{equation}
This objective encourages $Z_X$ to preserve task-relevant information about $Y$ while discarding redundant or irrelevant information from $X$.

\subsection{Graph Information Bottleneck}
The Graph Information Bottleneck (GIB) \citep{wu2020graph} principle extends IB to graph-structured data. 
Given graph data $\mathcal{D}$ and prediction target $Y$, GIB aims to learn node representations $Z_X$ that maximize predictive information about $Y$ while minimizing redundant information from the original graph $\mathcal{D}$. 
To approximate the optimal representation, 
% GIB relies on a local-dependence assumption: for a node $v$ within a certain hop range, information outside its local neighborhood is conditionally independent of $v$. Based on this assumption, 
GIB defines a Markovian dependence space $\Omega$ over the message-passing process. 
% At each layer $l$, node representations $Z_X^{(l)}$ are refined under a learned graph structure $Z_A^{(l)}$, where $\{Z_A^{(l)}\}_{1\le l\le L}$ can be adjusted from the original graph structure $\mathcal{D}$. 
The GIB objective is formulated as
%The graph information bottleneck (GIB) principle requires the node representation $Z_X$ to minimize the information from the graph-structured data $\mathcal{D}$ and maximize the information to prediction $Y$. To approximate the optimal $Z_X$, GIB relies on the local-dependence assumption: Given the node $v$ within a certain number of hops, the data in the rest of the graph will be independent of $v$. Based on this assumption, GIB defines the Markovian dependence space $\Omega$: The representation of each node will be refined by a graph structure $Z_A^{(l)}$, and $\{Z_A^{(l)}\}_{1\le l \le L}$ is changed by adjusting the graph structure $\mathcal{D}$. The objective function is
\begin{equation}
    \min_{\mathbb{P}(Z_X^{(L)}|\mathcal{D})\in \Omega} \text{GIB}(\mathcal{D},Y;Z_X^{(L)})\triangleq \left[-I(Y;Z^{(L)}_X)+\beta I(\mathcal{D};Z^{(L)}_X)
    \right].
\end{equation}
In this formulation, GIB purifies node representations by optimizing the mutual information
$I(Y;Z^{(L)}_X)$ and 
$I(\mathcal{D};Z^{(L)}_X)$, making $Z_X^{(L)}$ more predictive of $Y$ while being less dependent on redundant graph information. 
However, under multi-label settings, the central challenge is no longer merely to purify node representations from graph-structured redundancy. 
Instead, it is to selectively preserve informative label signals from source nodes while suppressing irrelevant label noise during message propagation. 
This motivates Multi-Label Graph Information Bottleneck, which models multi-label message passing as a constrained information transmission problem.

%In this formulation, GIB purifies the node representation by optimizing the $\mathbb{P}(Z^{(l)}_X|Z^{(l-1)}_X, Z^{(l)}_A)$ and $\mathbb{P}(Z^{(l)}_A|Z^{(l-1)}_X, A)$, ensuring the node representation $Z^{(L)}_X$ is more aligned with the prediction $Y$. Nevertheless, in the context of multi-label settings, our focus shifts to the purification of multi-label information from source nodes—this represents a fundamental difference between GIB and our proposed Multi-Label Graph Information Bottleneck.
\section{Multi-Label Graph Information Bottleneck}
\subsection{Deriving the Multi-Label Graph Information Bottleneck}
%We introduce a Multi-Label Graph Information Bottleneck (MLGIB) to minimize the message $Z_H$ from the source node data $\mathcal{D}_s$, where $\mathcal{D}_s=(A,X)$, and maximize the information to targeted node label $Y_{t}$. To optimal the principle, we adopt the local dependence assumption from GIB \citep{wu2020graph}, which posits that given the data associated with neighbors within a certain number of hops from a node $v$, the data in the remainder of the graph is independent of $v$.  Based on this assumption, we construct a solution space with Markovian dependencies, as follows:
We introduce the Multi-Label Graph Information Bottleneck (MLGIB) to formulate multi-label message passing as a constrained information transmission problem. 
Given source-node data $\mathcal{D}_s=(A,X)$, where $A$ denotes the graph structure from the source, MLGIB aims to learn messages $Z_H$ that preserve predictive information about the target-node labels $Y_t$ while suppressing redundant information from $\mathcal{D}_s$. 
To optimize this principle, we follow the local-dependence assumption in GIB \citep{wu2020graph}, which states that, conditioned on nodes within a $l$-hop neighborhood of $v$, the remaining graph data is independent of $v$. Based on this assumption, we construct a Markovian dependence space $\mathcal{M}$ for multi-label message passing:
\begin{equation}
\mathcal{M}=
    \begin{cases}
        A,Z_X^{(l)}\rightarrow Z_A^{(l)}\\
         Z_A^{(l)}, Z_X^{(l)}\rightarrow Z_H^{(l)}\\
          Z_H^{(l)},Z_A^{(l)}\rightarrow Z_X^{(l+1)}
    \end{cases},
\end{equation}
where $Z_X^{(0)}=X$, and $0\leq l\leq L$ denotes the message-passing layer. This dependence space characterizes the generation process of $\mathbb{P}(Z_X^{(l+1)}|\mathcal{D}_s)$. 
Starting from the source-node data $\mathcal{D}_s$, each layer first refines the message-passing structure $Z_A^{(l)}$, then extracts  messages $Z_H^{(l)}$, and finally updates the node representation $Z_X^{(l+1)}$. 
The key distinction of MLGIB lies in optimizing $\mathbb{P}(Z_H^{(l)}|\mathcal{D}_s)$, which can be decomposed into 
$\mathbb{P}(Z_A^{(l)}|\mathcal{D}_s)$ and 
$\mathbb{P}(Z_H^{(l)}|Z_A^{(l)},Z_X^{(l)})$. 
The former controls the message-passing paths, while the latter regulates message purity by selecting informative label signals and suppressing irrelevant noise. Following the intuition in Fig. \ref{fig:intro}, we define the MLGIB objective as
%where $Z_X^0=X$, and $0<l\leq L$ denotes the number of iterations. This dependency illustrates the modeling process of $\mathbb{P}(Z^{(l+1)}_X|D_s)$: starting from the raw data $D_s$, each iteration updates the message-passing paths $Z^{(l)}_A$, then filters out information $Z^{(l)}_H$ relevant to the labels, and updates the node features $Z^{(l+1)}_X$. The key innovation of MLGIB is optimizing $\mathbb{P}(Z^{(l)}_H|D_s)$, which consists of $\mathbb{P}(Z^{(l)}_A|D_s)$ and $\mathbb{P}(Z^{(l)}_H|Z^{(l)}_A, Z^{(l)}_X)$, responsible respectively for refining the message-passing weights and message purity. Finally, we design the objective function following the idea presented in Fig. \ref{fig:intro}:
\begin{equation}
    \min_{\mathbb{P}(Z^{(L)}_H|\mathcal{D}_s)\in \mathcal{M}}\textbf{MLGIB}(\mathcal{D}_s,Y_{t};Z^{(L)}_H) \triangleq \left[-I(Y_{t};Z^{(L)}_H)+\beta I(\mathcal{D}_s;Z^{(L)}_H)
    \right].
    \label{eq:objective_function}
\end{equation}
The first term encourages \textit{expressiveness} by maximizing the predictive information between the learned messages and target-node labels, while the second term promotes \textit{robustness} by limiting redundant information inherited from the source-node data.

% In this formulation, we need to optimize the distributions $\mathbb{P}(Z^{(l)}_H|Z^{(l)}_A,Z^{(l)}_X)$ and $\mathbb{P}(Z^{(l)}_A|A,Z^{(l)}_X)$, $l\in [1,L]$.
%In this formulation, the computation of $I(Y_{t};Z^{(L)}_H)$ and $I(\mathcal{D}_s;Z^{(L)}_H)$ is intractable. We introduce the variational methods, which are frequently used in traditional IB principle \citep{alemi2017deep} and GIB principle \citep{wu2020graph}, to optimize these two terms. A lower bound of $I(Y_{t};Z^{(L)}_H)$ and an upper bound of $I(\mathcal{D}_s;Z^{(L)}_H)$ are provided in Proposition \ref{proposition:lower_bound} and Proposition \ref{proposition:upper_bound}. More derivation details are presented in Appendix \ref{proof_supplement}.
However, directly computing $I(Y_t;Z_H^{(L)})$ and $I(\mathcal{D}_s;Z_H^{(L)})$ is generally intractable. 
Therefore, following variational techniques widely used in the IB principle \citep{alemi2017deep} and GIB \citep{wu2020graph}, we derive a tractable lower bound for $I(Y_t;Z_H^{(L)})$ and an upper bound for $I(\mathcal{D}_s;Z_H^{(L)})$. 
These bounds are given in Proposition \ref{proposition:lower_bound} and Proposition \ref{proposition:upper_bound}, respectively, with detailed derivations provided in Appendix \ref{proof_supplement}.
\begin{proposition}
   % For any distribution $\mathbb{Q}_1\bigl(Y_t \mid Z_H^{(L)}\bigr)$ and $\mathbb{Q}_2(Y_t)$, the lower bound of $I(Y_{t};Z^{(L)}_H)$ is
       For any variational distributions $\mathbb{Q}_1\bigl(Y_t \mid Z_H^{(L)}\bigr)$ and $\mathbb{Q}_2(Y_t)$, the mutual information $I(Y_t;Z_H^{(L)})$ admits the following lower bound:
    \begin{equation}
    % \resizebox{0.6\linewidth}{!}{$        \displaystyle
        I\bigl(Y_t; Z_H^{(L)}\bigr) \geq 1 + \mathbb{E}\left[ \log\frac{\mathbb{Q}_1\bigl(Y_t \mid Z_H^{(L)}\bigr)}{\mathbb{Q}_2(Y_t)} \right] - \mathbb{E}_{\mathbb{P}(Y_t)\mathbb{P}(Z_H^{(L)})}\left[ \frac{\mathbb{Q}_1\bigl(Y_t \mid Z_H^{(L)}\bigr)}{\mathbb{Q}_2(Y_t)} \right].
        \label{eq:lower_bound}
        % $}
    \end{equation}
    \label{proposition:lower_bound}
\end{proposition}

\begin{proposition}
%For any $\mathbb{Q}\bigl(Z_H^{(l)}\bigr)$, $l\in S_H$, and $\mathbb{Q}\bigl(Z_A^{(l)}\bigr)$, $l\in S_A$, based on the Markovian dependencies of $\mathcal{M}$,
For any variational distributions $\mathbb{Q}\bigl(Z_H^{(l)}\bigr)$ with $l\in S_H$ and $\mathbb{Q}\bigl(Z_A^{(l)}\bigr)$ with $l\in S_A$, under the Markovian dependence space $\mathcal{M}$, the mutual information $I(\mathcal{D}_s;Z_H^{(L)})$ admits the following upper bound:
\begin{equation}
% \resizebox{0.7\linewidth}{!}{$        \displaystyle
\begin{aligned}
    I\bigl(\mathcal{D}_s; Z_H^{(L)}\bigr) &\leq I\Bigl(\mathcal{D}_s; \bigl\{Z_H^{(l)}\bigr\}_{l \in S_H} \cup \bigl\{Z_A^{(l)}\bigr\}_{l \in S_A}\Bigr) \leq \sum_{l \in S_A} \textnormal{AIB}^{(l)} + \sum_{l \in S_H} \textnormal{HIB}^{(l)},\\
     \textnormal{AIB}^{(l)} &=  \mathbb{E}\left[ \log\frac{\mathbb{P}\bigl(Z_A^{(l)} \mid A, Z_X^{(l)}\bigr)}{\mathbb{Q}\bigl(Z_A^{(l)}\bigr)} \right],
     \textnormal{HIB}^{(l)} = \mathbb{E}\left[ \log\frac{\mathbb{P}\bigl(Z_H^{(l)} \mid Z_A^{(l)}, Z_X^{(l)}\bigr)}{\mathbb{Q}\bigl(Z_H^{(l)}\bigr)} \right],
    \end{aligned}
    % $}
    \label{eq:upper_bound}
\end{equation}
where $S_H$ and $S_A$ satisfy: (1) $S_H \subseteq \{0,1,\ldots,L\}$ and $S_H \neq \emptyset$; (2) $S_A = \{ \max(S_H), \ldots, L \}$.  
    \label{proposition:upper_bound}
\end{proposition}

% Next, to approximate the objective function, we model $\mathbb{Q}_1\bigl(Y_t \mid Z_H^{(L)}\bigr)$,$\mathbb{P}\bigl(Z_A^{(l)} \mid A, Z_X^{(l)}\bigr)$, and $\mathbb{P}\bigl(Z_H^{(l)} \mid Z_A^{(l)}, Z_X^{(l)}\bigr)$ and set $\mathbb{Q}\bigl(Z_H^{(l)}\bigr)$ and $\mathbb{Q}\bigl(Z_A^{(l)}\bigr)$ appropriately.

\subsection{Instantiating the MLGIB Principle}
%The preceding part approximates our objective function via variational upper and lower bounds of Proposition \ref{proposition:lower_bound} and \ref{proposition:upper_bound}. In this section, we instantiate these bounds, endowing our theory with practical deployability and thereby completing our theoretical framework. The instantiation method is shown in Fig. \ref{fig:model}.
% The previous section derives variational lower and upper bounds for MLGIB in Proposition \ref{proposition:lower_bound} and \ref{proposition:upper_bound}. 
In this section, we instantiate Proposition \ref{proposition:lower_bound} \& \ref{proposition:upper_bound} as tractable learning objectives, yielding an end-to-end differentiable message-passing architecture.  
The overall instantiation is illustrated in Fig. \ref{fig:model}.

\subsubsection{Instantiation of Proposition \ref{proposition:lower_bound} (Lower Bound)}

%We proceed with further derivations of each term in Eq. \ref{eq:lower_bound}, including
We start by instantiating the lower bound in Eq. \ref{eq:lower_bound},
% According to Eq. \ref{eq:lower_bound}, the lower bound consists of two terms: $\mathbb{E}\left[ \log\frac{\mathbb{Q}_1\bigl(Y_t \mid Z_H^{(L)}\bigr)}{\mathbb{Q}_2(Y_t)} \right]$ and $\mathbb{E}_{\mathbb{P}(Y_t)\mathbb{P}(Z_H^{(L)})}\left[ \frac{\mathbb{Q}_1\bigl(Y_t \mid Z_H^{(L)}\bigr)}{\mathbb{Q}_2(Y_t)} \right]$,
where for the first term, we have
\begin{equation}
% \resizebox{0.5\linewidth}{!}{$        \displaystyle
    \mathbb{E}\left[ \log\frac{\mathbb{Q}_1\bigl(Y_t \mid Z_H^{(L)}\bigr)}{\mathbb{Q}_2(Y_t)} \right]= \mathbb{E} \left[ \log \mathbb{Q}_1 \left( Y_t \middle| Z_H^{(L)} \right) \right] - \mathbb{E} \left[ \log \mathbb{Q}_2(Y_t) \right].
    \label{eq:q1/q2}
    % $}
\end{equation}
We set $\mathbb{Q}_2(Y_t)$ as the empirical label distribution $\mathbb{P}(Y_t)$. 
Then, the second term in Eq. \ref{eq:q1/q2} becomes
%Let $\mathbb{Q}_2(Y_t)$ be the distribution of ground-truth labels, ie., $\mathbb{Q}_2(Y_t)=\mathbb{P}(Y_t)$. From the second term in Eq. \ref{eq:q1/q2}, we have
\begin{equation}
\resizebox{0.89\linewidth}{!}{$        \displaystyle
    \mathbb{E} \left[ \log \mathbb{Q}_2(Y_t) \right]=\int_Y \int_{Z_H^{(L)}} \mathbb{P} \left( Y_t, Z_H^{(L)} \right) \log \mathbb{P}(Y_t) \, \mathrm{d}Z_H^{(L)} \, \mathrm{d}Y = \int_Y \mathbb{P}(Y_t) \log \mathbb{P}(Y_t) \, \mathrm{d}Y = \text{Const}.
    $}
\end{equation}
Since this term depends only on the ground-truth label distribution, it is independent of model parameters and can be treated as a constant. 
Therefore,
%$\int_Y \mathbb{P}(Y_t) \log \mathbb{P}(Y_t) \, \mathrm{d}Y$, consist of the distribution of ground-truth labels, can be treated as a constant that requires no fitting, and we have
\begin{equation}
% \resizebox{0.5\linewidth}{!}{$        \displaystyle
    \mathbb{E}\left[ \log\frac{\mathbb{Q}_1\bigl(Y_t \mid Z_H^{(L)}\bigr)}{\mathbb{Q}_2(Y_t)} \right]= \mathbb{E} \left[ \log \mathbb{Q}_1 \left( Y_t \middle| Z_H^{(L)} \right) \right] - \text{Const}.
    % $}
\end{equation}
The second term of Eq. \ref{eq:lower_bound} is
\begin{equation}
\resizebox{0.7\linewidth}{!}{$        \displaystyle
\begin{aligned}
    \mathbb{E}_{\mathbb{P}(Y_t)\mathbb{P}(Z_H^{(L)})}\left[ \frac{\mathbb{Q}_1\bigl(Y_t \mid Z_H^{(L)}\bigr)}{\mathbb{Q}_2(Y_t)} \right]&
    = \int_{Z_H^{(L)}} \int_Y \frac{\mathbb{Q}_1\left(Y_t \middle| Z_H^{(L)}\right)}{\mathbb{P}(Y_t)} \mathbb{P}(Y_t) \mathbb{P}\left(Z_H^{(L)}\right) \, \mathrm{d}Y \, \mathrm{d}Z_H^{(L)} 
    % \\&= \int_{Z_H^{(L)}} \int_Y \mathbb{Q}_1\left(Y_t \middle| Z_H^{(L)}\right) \mathbb{P}\left(Z_H^{(L)}\right) \, \mathrm{d}Y \, \mathrm{d}Z_H^{(L)} 
    \\& =\int_{Z_H^{(L)}} \mathbb{P}\left(Z_H^{(L)}\right) \, \mathrm{d}Z_H^{(L)} = 1.
\end{aligned}
$}
\end{equation}
Therefore, the lower bound in Proposition \ref{proposition:lower_bound} can be simplified as
\begin{equation}
    I(Y_t;Z_H^{(L)})\ge \mathbb{E} \left[ \log \mathbb{Q}_1 \left( Y_t \middle| Z_H^{(L)} \right) \right] - \text{Const}.
\end{equation}
%Based on the Monte Carlo sampling method, the expectation $\mathbb{E} \left[ \log \mathbb{Q}_1 \left( Y_t \middle| Z_H^{(L)} \right) \right]$ can be reduced to a binary cross-entropy (BCE) loss:
% We parameterize $Q_{1}\!\left(Y_{t}\mid Z_{H}^{(L)}\right)$ as an independent Bernoulli distribution over labels, whose negative log-likelihood reduces to the binary cross-entropy loss. 
% In practice, using Monte Carlo sampling, we instantiate this term as a multi-label prediction objective:
We parameterize $Q_{1}(Y_t \mid Z_H^{(L)})$ as an independent Bernoulli distribution over labels, reducing to:
\begin{equation}
    -I(Y_t;Z_H^{(L)})\rightarrow \sum\nolimits_{v\in \mathcal{V}}\text{BCE}(y_{v}, \sigma(\mathop{\text{Aggfunc}}(z_{h,u\rightarrow v}^{(L)}))),
\end{equation}
where BCE is the binary cross-entropy loss, $\mathcal{V}$ is the node set, $y_v$ is the multi-label ground-truth of node $v$, and $z_{h,u\rightarrow v}^{(L)}$ is the message passed from node $u$ to node $v$ at layer $L$. 
The sigmoid function $\sigma$ independently estimates the probability of each label. The aggregation function is defined as%$\mathcal{V}$ is the node set from the entire graph, $y_{v}$ is the instance of $Y_t$ for node $v$, and $z_{h,u\rightarrow v}^{(L)}$ is the instance of $Z_H^{(L)}$ for message from node $u$ to node $v$. The sigmoid function independently computes the probability for each class in a multi-label setting, and Aggfunc($\cdot$) is the function of graph aggregation, which can simply represent
\begin{equation}
    \mathop{\text{Aggfunc}}(z_{h,u\rightarrow v}^{(L)})=\sum\nolimits_{u\in \alpha_v^{(L)}}Wz_{h,u\rightarrow v}^{(L)},
\end{equation}
where $\alpha_v^{(L)}$, instantiated from $Z_A^{(L)}$, denotes the sampled message-passing neighborhood of node $v$, and $W$ is a learnable matrix.
Therefore, maximizing the lower bound enhances the expressiveness of MLGIB by encouraging the propagated messages to preserve predictive label information.
%where $z_{x,v}^{(L)}$ is the yet-to-be-updated node representation, and after an MLP, it becomes equivalent to the node's sent message, ie., $z_{x,v}^{(L)}=z_{h,u\rightarrow v}^{(L)}$.  $\alpha_{v}^{(L)}$ is the instance of $Z_A^{(L)}$ and it represents the sample set for node $v$. $W$ is the learnable weight.
% Note that even though node $v$ integrates global graph information, information flow is modulated by the weight $\alpha_{u,v}^{(L)}$. If a weight becomes zero, information from the corresponding node is entirely blocked.
\begin{figure*} 
    \centering
    \includegraphics[width=0.9\linewidth]{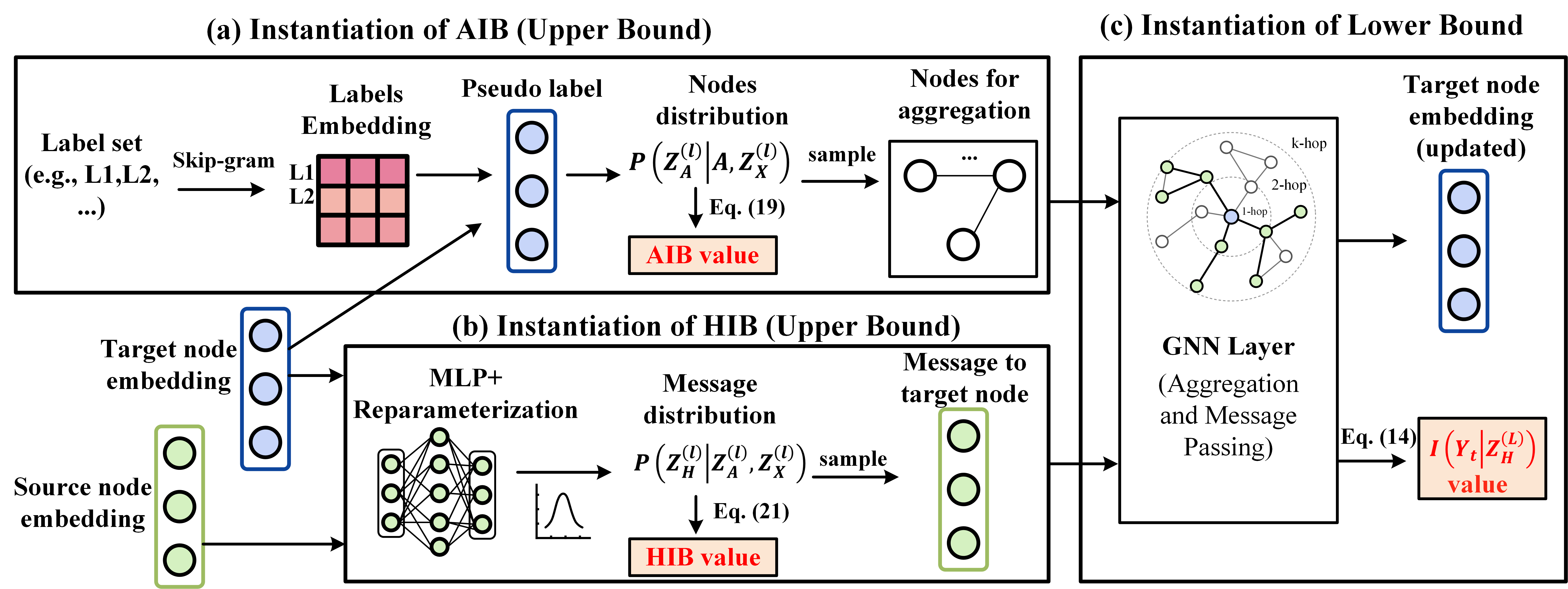}
    \caption{MLGIB is instantiated as a label-aware message-passing pipeline. (a) The instantiation of AIB learns label embeddings, derives pseudo labels, and samples label-relevant neighbors to constrain redundant structural information; (b) The instantiation of HIB parameterizes and samples transmitted messages to improve message purity.  (c) The selected neighbors and purified messages are then aggregated in a GNN layer, instantiating the lower-bound objective and producing updated target-node embeddings that preserve predictive label signals.}
    \label{fig:model} \vspace{-3mm}
\end{figure*}
\subsubsection{Instantiation of Proposition \ref{proposition:upper_bound} (Upper Bound)}
The upper bound of $I(
\mathcal{D}_s;Z_H^{(L)})$ consists of two terms: $\textnormal{AIB}^{(l)}$ for regulating message-passing paths and $\textnormal{HIB}^{(l)}$ for controlling message purity. %is decomposed into $\text{AIB}^{(l)}$ and $\text{HIB}^{(l)}$, where
\paragraph{Instantiation of AIB.}
The first term in Proposition \ref{proposition:upper_bound} is defined as
    $\text{AIB}^{(l)} = \mathbb{E}\left[ \log\frac{\mathbb{P}\bigl(Z_A^{(l)} \mid A, Z_X^{(l)}\bigr)}{\mathbb{Q}\bigl(Z_A^{(l)}\bigr)} \right]$,
%$Z_A^{(l)}$ is the sampled message paths. In multi‑label graphs, we choose to tie the message‑passing weights to label correlations. This implies that we need to sample more paths connecting label‑relevant nodes from the multi‑label graph, so that the probability $\mathbb{P}\bigl(Z_A^{(l)} \mid A, Z_X^{(l)}\bigr)$ is higher when labels are correlated. To achieve this, we start to model $Z_A^{(l)}$ by learning a group of label embeddings.
%The label embeddings are calculated by a simple skip-gram algorithm \citep{mikolov2013distributed}, which uses word frequency to learn the representation. Specifically, we sample a set of labels $L_\mathcal{V}$ from the training set data. For any label set $L_v$ of a node $v$, we randomly select a label $l_t$ as the target word, and select other labels $l_c$ as the positive context words, while a label $l_n$ is considered a negative sample if it is absent from the label set of any node that contains label  $l_t$. The loss of skip-gram is 
where $Z_A^{(l)}$ denotes the sampled message-passing paths. 
In multi-label graphs, message-passing paths should favor label-relevant neighbors. 
Therefore, we parameterize $\mathbb{P}\bigl(Z_A^{(l)} \mid A,Z_X^{(l)}\bigr)$ using label correlations, so that paths connecting label-relevant nodes are assigned higher probabilities.

% To capture label correlations, we first learn label embeddings with a skip-gram objective \citep{mikolov2013distributed}. 
% Given the label set $L_v$ of node $v$, we randomly select a target label $l_t$, treat the remaining labels $l_c\in L_v$ as positive contexts, and sample negative labels $l_n\notin L_v$. 
% The skip-gram loss is
% \begin{equation}
% \resizebox{0.89\linewidth}{!}{$        \displaystyle
%     \mathcal{L}_{\text{skip-gram}}=-\sum_{v\in \mathcal{V}}\sum_{l_c\in L_v}\log \sigma(\text{emb}(l_c)^\top \text{emb}(l_t))-\epsilon\sum_{v\in \mathcal{V}}\sum_{l_n\notin L_v}\log\sigma(-\text{emb}(l_n)^\top \text{emb}(l_t)),
%     \label{eq:skipgram}
%     $}
% \end{equation}

To capture label correlations, we learn label embeddings using a skip-gram objective \citep{mikolov2013distributed} in training data. 
Let $L_v$ be the label set of node $v$. For $l_t \in L_v$, we use $L_v \setminus \{l_t\}$ as positive contexts and draw negatives $l_n \in N(L_v)$, where $N(L_v)$ denotes $K$ labels sampled from outside $L_v$.
\begin{equation}
\resizebox{0.89\linewidth}{!}{$
\displaystyle
\mathcal{L}_{\text{skip-gram}} =
-\sum_{v\in \mathcal{V}}
\Bigg[
\sum_{l_c\in L_v \setminus \{l_t\}}
\log \sigma(\text{emb}(l_c)^\top \text{emb}(l_t))
+
\sum_{l_n \in N(L_v)}
\log \sigma(-\text{emb}(l_n)^\top \text{emb}(l_t))
\Bigg],
$}
\end{equation}
where $\text{emb}(\cdot)$ is the embedding function. This yields a dictionary $\mathcal{D}_L \in \mathbb{R}^{C\times d}$, containing embeddings for all labels, where $C$ is the number of labels and $d$ is the dimension. %A pseudo-label guided by label relevance is
For each node $v$, we construct a label-aware representation
\begin{equation}
    z_{p,v}^{(l)}=\text{MLP}(z_{x,v}^{(l)})^\top \mathcal{D}_L,\;\text{MLP}(z_{x,v}^{(l)})\in  \mathbb{R}^C,
    \label{eq:pseudo-label}
\end{equation}
%where $z_{x,v}^{(l)}$ is the instance of $Z_X^{(l)}$. This yields a message‑passing probability influenced by label relevance between node $v$ and node $u$:
where $z_{x,v}^{(l)}$ is the layer-$l$ node representation. 
The probability of selecting node $u$ as a message-passing neighbor of node $v$ is then defined as
\begin{equation}
% \resizebox{0.5\linewidth}{!}{$        \displaystyle
    \mathbb{P}\bigl(Z_A^{(l)}=(v,u) \mid A, Z_X^{(l)}\bigr)=\frac{\text{exp}(({z_{p,v}^{(l)})^\top z_{p,u}^{(l)}})}{\sum_{i\in \mathcal{N}_v}\text{exp}(({z_{p,v}^{(l)})^\top z_{p,i}^{(l)}})},
    \label{eq: message_passing_probability}
    % $}
\end{equation}
where $u\in \mathcal{N}_v$,  $\mathcal{N}_v$ is the 1-hop neighbor nodes of node $v$, and
% In this instantiation example, we simply let $\mathcal{N}_v$ be the set of 1-hop neighbors and 
$\left|\mathcal{N}_v\right|$ denotes the number of neighbors.
We further set $Q(Z_A^{(l)})$ to be a uniform distribution \citep{wu2020graph} for regularization, which yield
\begin{equation}
% \resizebox{0.6\linewidth}{!}{$        \displaystyle
    \text{AIB}^{(l)}\rightarrow \sum_{v\in \mathcal{V}}\sum_{u\in \mathcal{N}_v}\left(\text{log}\frac{\text{exp}(({z_{p,v}^{(l)})^\top z_{p,u}^{(l)}})}{\sum_{i\in \mathcal{N}_v}\text{exp}(({z_{p,v}^{(l)})^\top z_{p,i}^{(l)}})}-\text{log}\frac{1}{\left|\mathcal{N}_v\right|}\right).
    \label{eq:AIB_final}
    % $}
\end{equation}
%Next, we instantiate the remaining part of the upper bound on $I(\mathcal{D}_s,Z_H^{(L)})$, where
\paragraph{Instantiation of HIB.}
The second term in \ref{proposition:upper_bound} is defined as
    $\text{HIB}^{(l)} = \mathbb{E}\left[ \log\frac{\mathbb{P}\bigl(Z_H^{(l)} \mid Z_A^{(l)}, Z_X^{(l)}\bigr)}{\mathbb{Q}\bigl(Z_H^{(l)}\bigr)} \right]$.
Following the modeling strategy widely used in the Information Bottleneck literature \citep{alemi2017deep,wu2020graph,saxe2019information},
we start by modeling $\mathbb{Q}\bigl(Z_H^{(l)}\bigr)$ via a Gaussian mixture distribution for node $v$ and node $u$, ie., $\mathbb{Q}\bigl(Z_H^{(l)}\bigr)\rightarrow\sum\limits_{i=1}^mw_i\Phi(z_{h,v\rightarrow u}^{(l)};\mu_i;\sigma_i^2)$, where $w_i$, $\mu_i$ and $\sigma_i^2$ are learnable parameters, and $\Phi(\cdot)$ denotes a Gaussian density.
To enable stochastic optimization of $\mathbb{P}\bigl(Z_H^{(l)} \mid Z_A^{(l)}, Z_X^{(l)}\bigr)$, we model it as a Gaussian distribution parameterized by neural networks, and apply the reparameterization trick \citep{kingma2014auto} for sampling:
$z_{h,v\rightarrow u}^{(l)} = \mu_{v,u}^{(l)} + \sigma_{v,u}^{(l)} \cdot \epsilon,\; \epsilon \sim \mathcal{N}(0,I)$,
where
\begin{equation}
\mu_{v,u}^{(l)}=\text{MLP}_1(z_{x,v}^{(l)}\Vert z_{x,u}^{(l)}),\quad
\sigma_{v,u}^{2,(l)}=\text{MLP}_2(z_{x,v}^{(l)}\Vert z_{x,u}^{(l)}),
\end{equation}
and $\Vert$ denotes concatenation. By concatenating the two nodes and applying two separate MLPs, we obtain the mean and variance of the message sent from node $v$ to node $u$, yielding HIB as follows:
\begin{equation}
\resizebox{0.8\linewidth}{!}{$        \displaystyle
    \text{HIB}^{(l)}\rightarrow \sum_{v\in \mathcal{V}}\sum_{u\in \alpha_v^{(l)}}\bigg[\text{log}\Phi\left(z_{h,v\rightarrow u}^{(l)};\mu_{v,u}^{(l)};\sigma_{v,u}^{2,(l)}\right)-\text{log}\sum\limits_{i=1}^mw_i\Phi\left(z_{h,v\rightarrow u}^{(l)};\mu_i;\sigma_i^2\right)\bigg].
    \label{eq:HIB_final}
    $}
\end{equation}
Finally, substituting the instantiated results of AIB and HIB back into Eq. \ref{eq:upper_bound}.
% \begin{equation}
%     I(\mathcal{D}_s;Z_H^{(l)})\rightarrow\sum_{l \in S_A} \text{AIB}^{(l)} + \sum_{l \in S_H} \text{HIB}^{(l)}.
%     \label{eq:upper_bound_final}
% \end{equation}
By minimizing this upper bound, MLGIB constrains redundant information from the source graph, thereby improving robustness against irrelevant label noise during multi-label message passing.
\vspace{0mm}
\section{Experiment}
\subsection{Experiment Setup}
\paragraph{Dataset.}
We conduct experiments on four real-world multi-label datasets including DBLP \citep{akujuobi2019collaborative}, {BlogCatalog} \citep{lakshmanan2010knowledge}, PCG \citep{zhao2023multi} and {Delve} \citep{xiao2022semantic}. Details are provided in Appendix \ref{Dataset_supplement}.

\paragraph{Baseline Models.}
%We compare MLGIB with four categories of baseline methods: representative graph neural networks including GCN \citep{DBLP:conf/iclr/KipfW17} and GAT \citep{velivckovic2017graph}; the first application of the Information Bottleneck principle to graphs, ie., Graph Information Bottleneck (GIB) \citep{wu2020graph}; graph rewiring methods to address the over-squashing problem, including SDRF \citep{DBLP:conf/iclr/ToppingGC0B22}, FOSR \citep{karhadkar2023fosr}, and BORF \citep{nguyen2023revisiting}. and a state-of-the-art multi-label graph approach, CorGCN \citep{bei2025correlation}.
We compare MLGIB with four categories of baselines: representative GNNs, including GCN \citep{DBLP:conf/iclr/KipfW17} and GAT \citep{velivckovic2018graph}; the graph application of the IB principle, GIB \citep{wu2020graph}; over-squashing mitigation methods based on graph rewiring, including SDRF \citep{DBLP:conf/iclr/ToppingGC0B22}, FOSR \citep{karhadkar2023fosr}, and BORF \citep{nguyen2023revisiting}; and the state-of-the-art multi-label graph method CorGCN \citep{bei2025correlation}.

\paragraph{Evaluation Metrics.}
%Following prior work on multi-label graph models \citep{gao2019semi,zhou2021multi,bei2025correlation}, to comprehensively evaluate the performance of our multi-label graph model, we employed a suite of seven established metrics: Macro-AUC, Micro-AUC, Ranking Loss, Hamming Loss, Macro-AP, Micro-AP, and LRAP. More details about the datasets can be found in Appendix \ref{Metric_supplement}.
% Following prior multi-label graph studies \citep{gao2019semi,zhou2021multi,bei2025correlation}, 
We evaluate MLGIB using seven metrics: Macro-AUC, Micro-AUC, Ranking Loss, Hamming Loss, Macro-AP, Micro-AP, and LRAP. Details are provided in Appendix \ref{Metric_supplement}.

\subsection{Performance Evaluation}
%The results of the node classification experiments are summarized in Table \ref{tab:main_results}, where the best results are highlighted in bold. We make the following key observations:
Table \ref{tab:main_results} reports the node classification results, with best results in bold. We summarize two observations:
(1) \textbf{Superiority of MLGIB on multi‑label graphs.} 
% Across four datasets, MLGIB obtains the best results on most metrics. 
Its improvement over GIB shows that the information bottleneck principle needs to be tailored to multi-label message passing, while its advantage over CorGCN highlights that modeling label correlations alone is insufficient, and that preserving predictive information during propagation is equally important.%MLGIB outperforms all baseline models across multiple multi‑label datasets, notably surpassing both GIB and the state‑of‑the‑art model CorGCN. Exceeding GIB demonstrates that the information‑theoretic bottleneck principle is better applied to multi‑label graphs, while outperforming CorGCN validates the effectiveness of this application.
 (2) \textbf{Effectiveness of MLGIB against multi-label over-squashing.} Rewiring methods (SDRF, FOSR, and BORF) provide limited gains and may suffer scalability issues, suggesting that denser or modified connectivity can amplify label noise in multi-label graphs. By contrast, MLGIB regulates information flow without altering the graph structure, achieving a better balance between expressiveness and robustness. %Graph rewiring methods, like SDRF, FORF and BORF, on multi‑label graphs have not led to significant performance improvements. This is because information flow in multi‑label graphs is more complex: altering graph structure can easily introduce noise and reduce structural information. In contrast, our method controls and purifies the information flow without modifying the graph structure, effectively resolving the over‑squashing problem.
\begin{table*}[htbp]
\tiny
\setlength{\tabcolsep}{2pt}
  \centering
  \vspace{0mm}
  \caption{Node classification results on multiple real‑world datasets using various metrics.}
    \begin{threeparttable}
    \begin{tabular}{ccccccccc}
    \toprule
    Dataset & Model & Macro-AUC & Micro-AUC & Ranking Loss & Hamming Loss & Macro-AP & Micro-AP & LRAP \\
    \midrule
    \multirow{8}[2]{*}{DBLP} & GCN   & 0.9485 ± 0.0005 & 0.9527 ± 0.0004 & 0.0510 ± 0.0004 & 0.0883 ± 0.0007 & 0.9049 ± 0.0007 & 0.9214 ± 0.0004 & 0.9447 ± 0.0005 \\
          & GAT   & 0.9501 ± 0.0023 & 0.9540 ± 0.0016 & 0.0575 ± 0.0009 & 0.0944 ± 0.0014 & 0.8901 ± 0.0043 & 0.9123 ± 0.0014 & 0.9387 ± 0.0008 \\
          & GIB   & 0.9478 ± 0.0026 & 0.9517 ± 0.0025 & 0.0636 ± 0.0039 & 0.0989 ± 0.0040 & 0.8868 ± 0.0066 & 0.9079 ± 0.0055 & 0.9319 ± 0.0047 \\
          & FOSR  & 0.9486 ± 0.0006 & 0.9527 ± 0.0005 & 0.0515 ± 0.0009 & 0.0886 ± 0.0010 & 0.9052 ± 0.0008 & 0.9215 ± 0.0008 & 0.9442 ± 0.0010 \\
          & SDRF  & 0.9487 ± 0.0005 & 0.9528 ± 0.0004 & 0.0509 ± 0.0004 & 0.0884 ± 0.0005 & 0.9050 ± 0.0005 & 0.9215 ± 0.0003 & 0.9448 ± 0.0005 \\
          & BORF  & 0.9477 ± 0.0005 & 0.9520 ± 0.0004 & 0.0520 ± 0.0005 & 0.0885 ± 0.0007 & 0.9031 ± 0.0008 & 0.9199 ± 0.0005 & 0.9437 ± 0.0007 \\
          & CORGCN & 0.9504 ± 0.0004 & 0.9545 ± 0.0004 & 0.0699 ± 0.0005 & 0.0983 ± 0.0009 & 0.8926 ± 0.0016 & 0.9134 ± 0.0010 & 0.9226 ± 0.0006 \\
          & \textbf{MLGIB(ours)} & \textbf{0.9699 ± 0.0008} & \textbf{0.9716 ± 0.0007} & \textbf{0.0441 ± 0.0012} & \textbf{0.0664 ± 0.0011} & \textbf{0.9356 ± 0.0018} & \textbf{0.9464 ± 0.0012} & \textbf{0.9531 ± 0.0012} \\
    \midrule
    \multirow{8}[2]{*}{BlogCatalog} & GCN   & 0.4617 ± 0.0017 & 0.6348 ± 0.0037 & 0.2608 ± 0.0037 & 0.0358 ± 0.0003 & 0.0360 ± 0.0005 & 0.0558 ± 0.0022 & 0.2604 ± 0.0110 \\
          & GAT   & 0.5356 ± 0.0135 & 0.7435 ± 0.0026 & 0.2550 ± 0.0025 & 0.0353 ± 0.0003 & 0.0497 ± 0.0065 & 0.1039 ± 0.0163 & 0.2747 ± 0.0114 \\
          & GIB   & 0.5229 ± 0.0105 & 0.7394 ± 0.0005 & 0.2579 ± 0.0004 & 0.0355 ± 0.0000 & 0.0463 ± 0.0045 & 0.0906 ± 0.0029 & 0.2645 ± 0.0008 \\
          & FOSR  & 0.4619 ± 0.0016 & 0.6333 ± 0.0024 & 0.2618 ± 0.0036 & 0.0357 ± 0.0002 & 0.0359 ± 0.0003 & 0.0558 ± 0.0026 & 0.2556 ± 0.0119 \\
          & SDRF  & 0.4621 ± 0.0019 & 0.6347 ± 0.0043 & 0.2615 ± 0.0037 & 0.0359 ± 0.0005 & 0.0361 ± 0.0005 & 0.0558 ± 0.0025 & 0.2572 ± 0.0134 \\
          & BORF  & OOM   & OOM   & OOM   & OOM   & OOM   & OOM   & OOM \\
          & CORGCN & 0.5207 ± 0.0095 & 0.7384 ± 0.0004 & 0.2588 ± 0.0004 & 0.0355 ± 0.0000 & 0.0422 ± 0.0010 & 0.0884 ± 0.0008 & 0.2638 ± 0.0003 \\
          & \textbf{MLGIB(ours)} & \textbf{0.5627 ± 0.0148} & \textbf{0.7527 ± 0.0022} & \textbf{0.2453 ± 0.0022} & \textbf{0.0347 ± 0.0004} & \textbf{0.0691 ± 0.0080} & \textbf{0.1430 ± 0.0151} & \textbf{0.3049 ± 0.0110} \\
    \midrule
    \multirow{8}[2]{*}{PCG} & GCN   & 0.4676 ± 0.0051 & 0.6265 ± 0.0025 & 0.2634 ± 0.0034 & 0.1358 ± 0.0004 & 0.1491 ± 0.0025 & 0.2256 ± 0.0049 & 0.4826 ± 0.0058 \\
          & GAT   & 0.5220 ± 0.0206 & 0.6803 ± 0.0136 & 0.2898 ± 0.0124 & 0.1431 ± 0.0035 & 0.1501 ± 0.0083 & 0.2326 ± 0.0146 & 0.4557 ± 0.0139 \\
          & GIB   & 0.5195 ± 0.0085 & 0.6927 ± 0.0086 & 0.2783 ± 0.0081 & 0.1364 ± 0.0006 & 0.1467 ± 0.0029 & 0.2496 ± 0.0072 & 0.4689 ± 0.0072 \\
          & FOSR  & 0.4738 ± 0.0068 & 0.6299 ± 0.0037 & 0.2621 ± 0.0041 & 0.1356 ± 0.0007 & 0.1508 ± 0.0029 & 0.2293 ± 0.0046 & 0.4838 ± 0.0055 \\
          & SDRF  & 0.4677 ± 0.0052 & 0.6266 ± 0.0025 & 0.2634 ± 0.0034 & 0.1358 ± 0.0004 & 0.1492 ± 0.0026 & 0.2258 ± 0.0049 & 0.4828 ± 0.0057 \\
          & BORF  & 0.5122 ± 0.0038 & 0.6443 ± 0.0024 & 0.2817 ± 0.0024 & 0.1381 ± 0.0006 & 0.1459 ± 0.0025 & 0.2217 ± 0.0042 & 0.4606 ± 0.0042 \\
          & CORGCN & 0.5571 ± 0.0177 & 0.7083 ± 0.0053 & 0.2647 ± 0.0013 & 0.1357 ± 0.0001 & 0.1797 ± 0.0075 & 0.2777 ± 0.0051 & 0.4816 ± 0.0033 \\
          & \textbf{MLGIB(ours)} & \textbf{0.6005 ± 0.0051} & \textbf{0.7252 ± 0.0021} & \textbf{0.2619 ± 0.0024} & \textbf{0.1356 ± 0.0001} & \textbf{0.2079 ± 0.0061} & \textbf{0.2951 ± 0.0055} & \textbf{0.4844 ± 0.0052} \\
    \midrule
    % \multirow{8}[2]{*}{Humloc} & GCN   & 0.6195 ± 0.0143 & 0.8214 ± 0.0046 & 0.1544 ± 0.0040 & \textbf{0.0777 ± 0.0008} & 0.1638 ± 0.0108 & 0.3585 ± 0.0115 & 0.5991 ± 0.0070 \\
    %       & GAT   & 0.5681 ± 0.0312 & 0.8074 ± 0.0160 & 0.1889 ± 0.0143 & 0.0875 ± 0.0043 & 0.1408 ± 0.0180 & 0.2866 ± 0.0344 & 0.5509 ± 0.0207 \\
    %       & GIB   & 0.5767 ± 0.0279 & 0.8278 ± 0.0086 & 0.1780 ± 0.0076 & 0.0823 ± 0.0015 & 0.1307 ± 0.0130 & 0.3020 ± 0.0309 & 0.5522 ± 0.0185 \\
    %       & FOSR  & 0.6208 ± 0.0144 & 0.8220 ± 0.0049 & 0.1543 ± 0.0039 & 0.0778 ± 0.0008 & 0.1643 ± 0.0110 & 0.3589 ± 0.0120 & 0.5990 ± 0.0062 \\
    %       & SDRF  & 0.6193 ± 0.0145 & 0.8213 ± 0.0047 & 0.1544 ± 0.0041 & 0.0778 ± 0.0008 & 0.1638 ± 0.0108 & 0.3585 ± 0.0115 & 0.5993 ± 0.0071 \\
    %       & BORF  & 0.6116 ± 0.0130 & 0.8168 ± 0.0049 & 0.1607 ± 0.0032 & 0.0782 ± 0.0012 & 0.1533 ± 0.0082 & 0.3489 ± 0.0112 & 0.5884 ± 0.0088 \\
    %       & CORGCN & 0.5905 ± 0.0188 & 0.8335 ± 0.0039 & 0.1726 ± 0.0045 & 0.0815 ± 0.0012 & 0.1571 ± 0.0065 & 0.3499 ± 0.0209 & 0.5578 ± 0.0136 \\
    %       & \textbf{MLGIB(ours)} & \textbf{0.7093 ± 0.0223} & \textbf{0.8643 ± 0.0023} & \textbf{0.1410 ± 0.0022} & 0.0790 ± 0.0021 & \textbf{0.2090 ± 0.0086} & \textbf{0.4115 ± 0.0180} & \textbf{0.6250 ± 0.0074} \\
    % \midrule
    \multirow{4}[2]{*}{Delve} & GCN   & 0.7110 ± 0.0075 & 0.8098 ± 0.0042 & 0.0091 ± 0.0003 & 0.0069 ± 0.0000 & 0.0699 ± 0.0027 & 0.1195 ± 0.0030 & 0.9713 ± 0.0007 \\
          & GAT   & 0.8931 ± 0.0020 & 0.9263 ± 0.0009 & 0.0084 ± 0.0001 & 0.0065 ± 0.0000 & 0.1050 ± 0.0032 & 0.1854 ± 0.0030 & 0.9740 ± 0.0003 \\
          % & GIB   & OOM   & OOM   & OOM   & OOM   & OOM   & OOM   & OOM  \\
          & CORGCN & 0.9541 ± 0.0011 & 0.9686 ± 0.0003 & 0.0039 ± 0.0000 & 0.0061 ± 0.0000 & 0.2548 ± 0.0062 & 0.3404 ± 0.0018 & 0.9854 ± 0.0001 \\
          & \textbf{MLGIB(ours)} & \textbf{0.9867 ± 0.0003} & \textbf{0.9903 ± 0.0001} & \textbf{0.0012 ± 0.0000} & \textbf{0.0056 ± 0.0001} & \textbf{0.4414 ± 0.0055} & \textbf{0.5334 ± 0.0052} & \textbf{0.9940 ± 0.0001} \\
    \bottomrule
    \end{tabular}%
        \begin{tablenotes}
            \item[] ``OOM'' stands for Out Of Memory. The implementations provided of FOSR, SDRF, BORF and GIB are not adapted for the large graph Delve. \vspace{0mm}
\end{tablenotes}
\end{threeparttable}
  \label{tab:main_results}%
\end{table*}%

%To validate the expressiveness of MLGIB in preserving informative label signals under multi-label interference, we conduct fine-grained analyses from both label-level performance and cross-setting generalization perspectives.
%\subsubsection{Expressiveness Analysis on Multi-label Label-wise Performance}
%In this part, we conduct experiments on the three datasets—Delve, DBLP, and Humloc. As shown in Fig. \ref{fig:each_label_eval}, the results demonstrate that our model outperforms the SOTA model CorGCN on the vast majority of labels. This result indicates that MLGIB is able to \textbf{consistently preserve predictive signals across different labels}, rather than overfitting to dominant or highly correlated labels. Such behavior reflects stronger expressiveness in disentangling informative label-specific signals from irrelevant ones during message passing. Furthermore, the performance gains are not limited to a subset of labels, but are broadly observed across the label space, demonstrating that MLGIB achieves both effectiveness and stability in multi-label representation learning.
% \begin{figure}
%     \centering
%     \includegraphics[width=0.8\linewidth]{img/label_auc.pdf}
%     \caption{Performance evaluation (AUC) is conducted for each label on the Delve, DBLP, and Humloc datasets, respectively.}
%     \label{fig:each_label_eval}
% \end{figure}

\subsection{Expressiveness Analysis}

To evaluate whether MLGIB can preserve informative label signals under multi-label interference, we analyze its expressiveness from two complementary perspectives: label-wise expressiveness on multi-label graphs and cross-setting expressiveness on single-label graphs.

\subsubsection{Label-wise Expressiveness on Multi-label Graphs}

We conduct label-wise evaluations on Delve, DBLP, and PCG. As shown in Fig. \ref{fig:each_label_eval}, MLGIB outperforms CorGCN on most labels, showing that it consistently preserves predictive signals across labels instead of only improving dominant or highly correlated ones. These results indicate stronger expressiveness in extracting label-specific informative signals from noisy multi-label propagation.

\subsubsection{Cross-setting Expressiveness on Single-label Graphs}

We evaluate MLGIB on the single-label PubMed dataset \citep{sen2008collective}, where nodes represent biomedical articles and edges denote citation links. Details are provided in Appendix \ref{Dataset_supplement}.
As shown in Fig. \ref{fig: siglabel}, MLGIB does not degrade in the single-label setting and even outperforms GCN and CorGCN. This demonstrates that MLGIB enhances the intrinsic expressiveness of message passing rather than overfitting to multi-label-specific patterns.
%\subsubsection{Expressiveness Analysis of Generalization to Single-label Graphs}
%In this section, we conduct experiments on the single‑label PubMed dataset \citep{sen2008collective} to evaluate the generalizability of our model on single‑label data. PubMed is a biomedical literature dataset where each article is a node and citation relationships form the edges. As shown in Fig. \ref{fig: siglabel}, our model does not exhibit significant degradation on single‑label datasets, and even outperforms GCN and the multi‑label SOTA model CorGCN. This suggests that the proposed information bottleneck formulation does not rely on specific multi-label assumptions, but instead learns more generalizable representations. This result further supports that \textbf{MLGIB enhances the intrinsic expressiveness of message passing}, rather than overfitting to multi-label-specific patterns.
% \begin{figure}
%     \centering
%     \includegraphics[width=0.4\linewidth]{img/radar_chart.pdf}
%     \caption{Generalization performance on the single‑label dataset PubMed.}
%     \label{fig: siglabel}
% \end{figure}

\subsection{Robustness Analysis}
%We next evaluate the robustness of MLGIB under challenging conditions where irrelevant label signals are amplified.
We further evaluate the robustness of MLGIB under challenging conditions where irrelevant label signals are likely to be amplified.
%\subsubsection{Robustness Analysis on Label Correlation}
%In multi-label node classification tasks, the label correlation among neighbors determines the classification difficulty. Meanwhile, lower label correlation is more likely to introduce the over‑squashing problem, as irrelevant label information will compress relevant label information in narrow channels. In this experiment, we calculate the average Jaccard similarity \citep{jaccard1912distribution} between nodes' labels and their neighbors' labels, and analyze its impact on the final classification performance. The Jaccard similarity reflects the similarity of label sets between two nodes. As shown in Fig. \ref{fig:jaccard_similarity}, node classification performance declines as the Jaccard similarity with neighbors decreases, while MLGIB exhibits a slower degradation trend and more stable performance fluctuations, \textbf{demonstrating that MLGIB can greatly mitigate the disturbances caused by decreasing label correlation and alleviate the over‑squashing problem.}
\subsubsection{Robustness under Low Label Correlation}

In multi-label node classification, lower label correlation between neighboring nodes usually indicates higher aggregation noise, as irrelevant label signals are more likely to interfere with predictive ones. 
To quantify this effect, we compute the average Jaccard similarity \citep{jaccard1912distribution} between each node's label set and those of its neighbors.
As shown in Fig. \ref{fig:jaccard_similarity}, on the large graph Delve, classification performance generally decreases as neighbor-label similarity becomes lower. 
However, MLGIB exhibits a slower performance degradation and more stable fluctuations than the baselines, demonstrating its robustness against weak label correlation and its ability to suppress irrelevant label noise during message passing.
\begin{figure}[t]
\centering
\begin{subfigure}[t]{0.6\linewidth}
    \centering
\includegraphics[width=\linewidth]{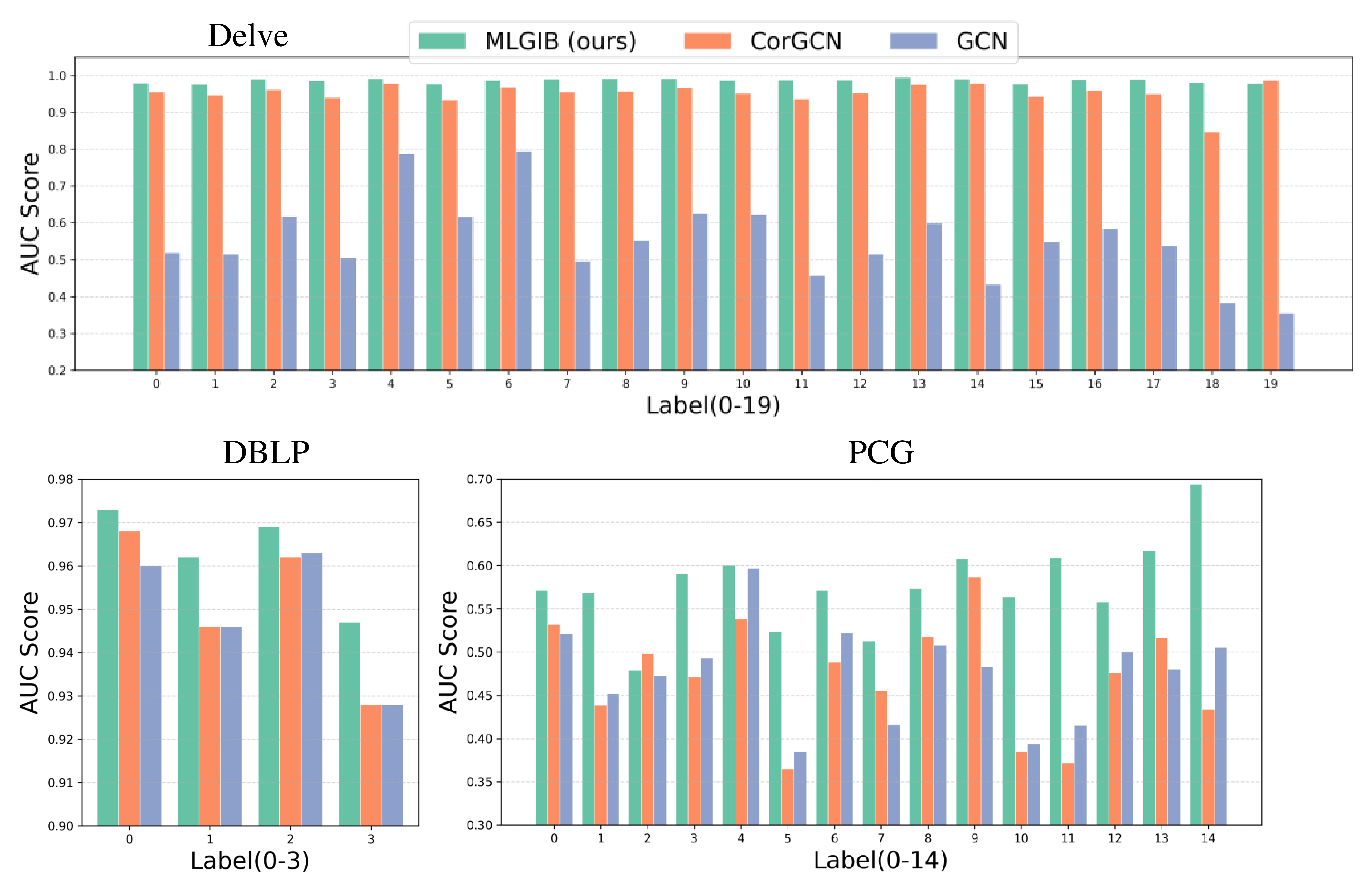}
    \caption{ Label-wise evaluations.}
    \label{fig:each_label_eval}
\end{subfigure}
\hfill
\begin{subfigure}[t]{0.3\linewidth}
    \centering
    \includegraphics[width=\linewidth]{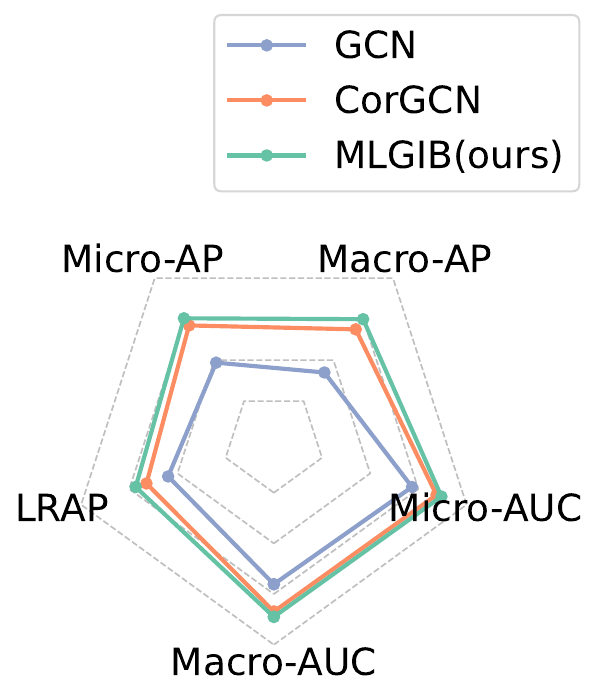}
    \caption{Single-label dataset}
    \label{fig: siglabel}
\end{subfigure}
\caption{(a) Expressiveness Analysis (AUC) is conducted for each label on the Delve, DBLP, and PCG datasets, respectively. (b) Cross-setting Expressiveness on Single-label Graphs PubMed.} \vspace{0mm}
\label{fig:baseline}
\end{figure}
\begin{figure}
    \centering
    \includegraphics[width=1\linewidth]{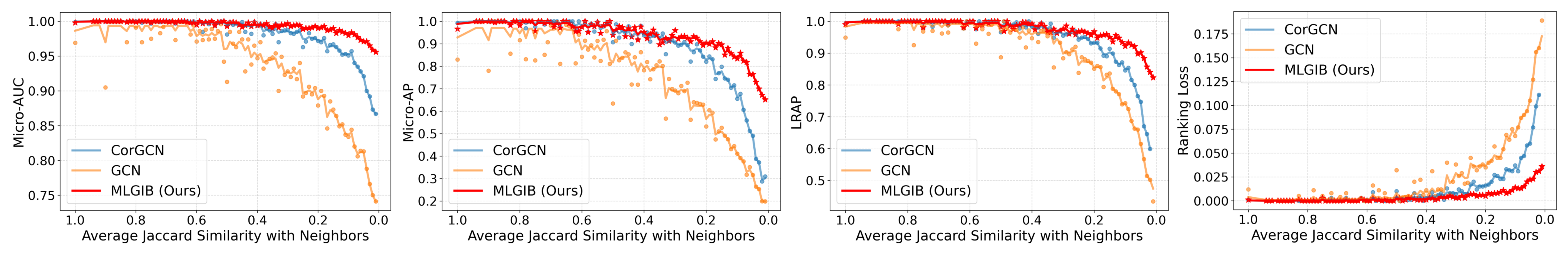}
    \caption{Impact of decreased Jaccard similarity with neighbor nodes on performance. Dots denote true values and lines denote fitted trends.} \vspace{0mm}
    \label{fig:jaccard_similarity}
\end{figure}

%\subsubsection{Robustness Analysis on Node Degree}
\subsubsection{Robustness under High Node Degree}

%Increasing node degree presents a dual effect. On the one hand, a larger neighborhood provides more information that can benefit classification; on the other hand, it also introduces more noise, which increases classification disturbance and triggers the over‑squashing problem. As shown in Fig. \ref{fig:degree}, as node degree increases, LRAP decreases while Hamming Loss rises, indicating a drop in the model’s confidence in its classification results. Our model, however, shows significantly less susceptibility to the disturbance caused by high degrees than CorGCN and GCN, \textbf{demonstrating that MLGIB’s strong filtering capability on multiple information sources successfully mitigates over‑squashing.}
Node degree has a dual effect on multi-label message passing. 
A larger neighborhood may provide richer context, but also introduce more irrelevant messages, increasing aggregation noise and aggravating over-squashing. 
As shown in Fig. \ref{fig:degree}, LRAP decreases and Hamming Loss increases with node degree of Delve dataset, indicating that high-degree nodes are more vulnerable to noisy aggregation. 
Compared with baselines, MLGIB shows much smaller performance degradation, suggesting that its information-filtering mechanism improves robustness to noisy and excessively large neighborhoods.

\begin{figure}
    \centering 
    \includegraphics[width=0.6\linewidth]{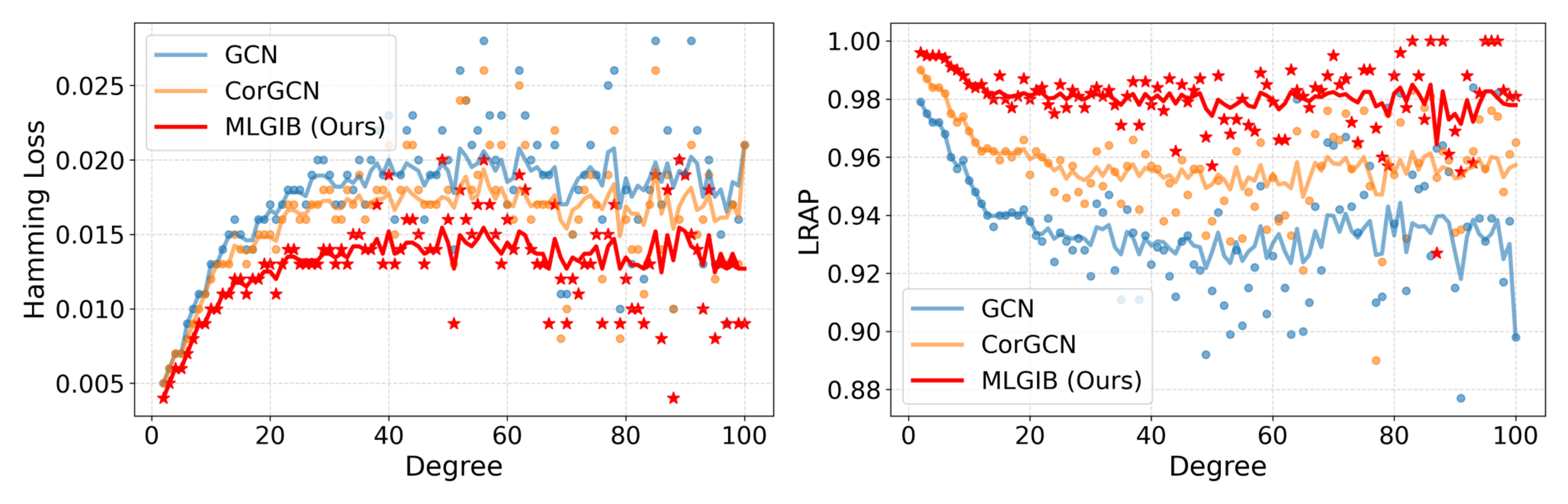}
    \caption{Impact of increased node degree on performance, where dots and lines follow Fig. \ref{fig:jaccard_similarity}.} \vspace{0mm}
    \label{fig:degree}
\end{figure}

\subsubsection{Robustness Analysis on Structural Perturbations}
Beyond inherent graph properties, we further evaluate robustness by actively injecting structural noise. 
We conduct this experiment on DBLP, whose clear structural semantics make it suitable for assessing the impact of perturbations on multi-label aggregation. 
Specifically, we randomly add edges according to a proportion of the original edge count, ranging from 0 to 1; when the proportion reaches 1, the graph contains twice as many edges as the original. 
As shown in Fig. \ref{fig:noise}, all models degrade as more random edges are added, but MLGIB exhibits the slowest degradation. 
This confirms that MLGIB is robust to structural perturbations by filtering irrelevant information from noisy edges, thereby mitigating multi-label over-squashing caused by rapidly expanding information sources.
%Unlike inherent graph properties such as label correlation and node degree, we next actively introduce structural perturbation noise to the graph nodes. We select the DBLP dataset for this experiment. Compared to noisier multi‑label datasets like BlogCatalog, DBLP possesses more rigorous structural semantics, which can more sensitively reveal how adversarial perturbations undermine the multi‑label aggregation mechanism. We randomly add edges to the DBLP graph according to a proportion of its original edge count, ranging from 0 to 1. When the proportion reaches 1, the graph contains twice the original number of edges. As shown in the Fig. \ref{fig:noise}, the performance of all models declines to varying degrees as the number of added edges increases, but our model exhibits the slowest degradation, demonstrating that MLGIB possesses strong robustness. This capability is exactly an effective means of addressing the over‑squashing problem caused by exponential information growth.
\begin{figure}
    \centering 
    \includegraphics[width=1\linewidth]{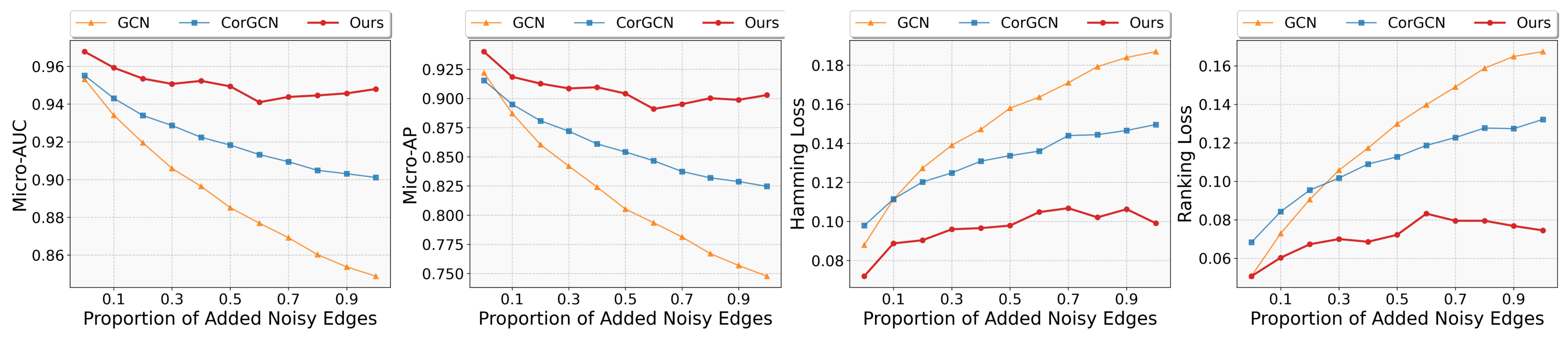}
    \caption{Impact of added noisy edges on performance.}
    \label{fig:noise}
\end{figure}

\begin{table*}[tb]
\setlength{\tabcolsep}{2.5pt}
\tiny
  \centering
  \caption{The results of ablation study.}
    \begin{tabular}{ccccccccc}
    \toprule
    Dataset & Model & Macro-AUC & Micro-AUC & Ranking Loss & Hamming Loss & Macro-AP & Micro-AP & LRAP \\
    \midrule
    \multirow{3}[2]{*}{DBLP} & MLGIB & \textbf{0.9699 ± 0.0008} & \textbf{0.9716 ± 0.0007} & \textbf{0.0441 ± 0.0012} & \textbf{0.0664 ± 0.0011} & \textbf{0.9356 ± 0.0018} & \textbf{0.9464 ± 0.0012} & \textbf{0.9531 ± 0.0012} \\
          & w/o AIB Instantiation  &0.9589 ± 0.0023  &0.9616 ± 0.0018  &0.0588 ± 0.0026  &0.0852 ± 0.0028  &0.9129 ± 0.0048  &0.9273 ± 0.0034  &0.9366 ± 0.0030  \\
          & w/o HIB Instantiation &0.9458 ± 0.0055  &0.9508 ± 0.0046  &0.0581 ± 0.0036  &0.0957 ± 0.0029  &0.8861 ± 0.0082  &0.9095 ± 0.0058  &0.9378 ± 0.0035  \\
    \midrule
    \multirow{3}[2]{*}{BlogCatalog} & MLGIB & \textbf{0.5627 ± 0.0148} & \textbf{0.7527 ± 0.0022} & \textbf{0.2453 ± 0.0022} & \textbf{0.0347 ± 0.0004} & \textbf{0.0691 ± 0.0080} & \textbf{0.1430 ± 0.0151} & \textbf{0.3049 ± 0.0110} \\
          & w/o AIB Instantiation  &0.5116 ± 0.0065  &0.7387 ± 0.0004  &0.2582 ± 0.0001  &0.0355 ± 0.0000  &0.0427 ± 0.0029  &0.0888 ± 0.0013  &0.2637 ± 0.0005  \\
          & w/o HIB Instantiation &0.5013 ± 0.0153  &0.7202 ± 0.0053  &0.2749 ± 0.0054  &0.0355 ± 0.0000  &0.0404 ± 0.0031  &0.0819 ± 0.0035  &0.2520 ± 0.0090  \\
    \midrule
    \multirow{3}[1]{*}{PCG} & MLGIB & \textbf{0.6005 ± 0.0051} & \textbf{0.7252 ± 0.0021} & \textbf{0.2619 ± 0.0024} & \textbf{0.1356 ± 0.0001} & \textbf{0.2079 ± 0.0061} & \textbf{0.2951 ± 0.0055} & \textbf{0.4844 ± 0.0052} \\
          & w/o AIB Instantiation  &0.5475 ± 0.0144  &0.7079 ± 0.0023  &0.2678 ± 0.0015  &0.1356 ± 0.0003  &0.1610 ± 0.0076  &0.2718 ± 0.0072  &0.4775 ± 0.0055  \\
          & w/o HIB Instantiation &0.5233 ± 0.0101  &0.6629 ± 0.0105  &0.3051 ± 0.0092  &0.1477 ± 0.0044  &0.1488 ± 0.0029  &0.2174 ± 0.0083  &0.4428 ± 0.0099  \\
    \midrule
    % \multirow{3}[1]{*}{Humloc} & MLGIB & \textbf{0.7093 ± 0.0223} & \textbf{0.8643 ± 0.0023} & \textbf{0.1410 ± 0.0022} & \textbf{0.0790 ± 0.0021} & \textbf{0.2090 ± 0.0086} & \textbf{0.4115 ± 0.0180} & \textbf{0.6250 ± 0.0074} \\
    %       & w/o $I(Y_t;Z_H^{(L)})$  & 0.5256 ± 0.0204 & 0.5739 ± 0.0841 & 0.4431 ± 0.0921 & 0.4347 ± 0.0724 & 0.1000 ± 0.0068 & 0.1094 ± 0.0375 & 0.2684 ± 0.0850 \\
    %       & w/o $I(\mathcal{D}_s;Z_H^{(l)})$ & 0.6038 ± 0.0289 & 0.8336 ± 0.0063 & 0.1729 ± 0.0070 & 0.0817 ± 0.0008 & 0.1374 ± 0.0127 & 0.3201 ± 0.0254 & 0.5539 ± 0.0177 \\
    % \midrule
    \multirow{3}[2]{*}{Delve} & MLGIB & \textbf{0.9867 ± 0.0003} & \textbf{0.9903 ± 0.0001} & \textbf{0.0012 ± 0.0000} & \textbf{0.0056 ± 0.0001} & \textbf{0.4414 ± 0.0055} & \textbf{0.5334 ± 0.0052} & \textbf{0.9940 ± 0.0001} \\
          & w/o AIB Instantiation  &0.9493 ± 0.0037  &0.9711 ± 0.0007  &0.0024 ± 0.0001  &0.0063 ± 0.0000  &0.2578 ± 0.0044  &0.3602 ± 0.0042  &0.9906 ± 0.0002  \\
          & w/o HIB Instantiation &0.8825 ± 0.0040  &0.9177 ± 0.0029  &0.0082 ± 0.0002  &0.0065 ± 0.0000  &0.1050 ± 0.0046  &0.1786 ± 0.0063  &0.9738 ± 0.0005  \\
    \bottomrule
    \end{tabular} \vspace{0mm}
  \label{tab:ablation_study}%
\end{table*}%
\subsection{Ablation Study}

%In this section, we perform an ablation study on our proposed framework MLGIB by removing the instantiation of $I(Y_t;Z_H^{(L)})$ and $I(\mathcal{D}_s;Z_H^{(l)})$, respectively. As shown in Table \ref{tab:ablation_study}, the following conclusions can be drawn:
%\begin{itemize}
 %   \item \textbf{Impact of Module Removal on Performance.} Removing any module results in a noticeable decline in performance, underscoring the effectiveness of our proposed modules.
 %   \item \textbf{Significant Performance Reduction without the instantiation of $I(Y_t;Z_H^{(L)})$.} Because this part requires the transmitted information to align with the target node’s labels, it directly affects the downstream node‑classification task, and therefore has a significant impact on performance.
 %   \item \textbf{Performance Decline without the instantiation of $I(\mathcal{D}_s;Z_H^{(l)})$.}This stage pertains to the effective filtering of information, requiring purification and filtering of the source node’s structure and features, thus also affecting the final performance.
%\end{itemize}
We conduct an ablation study to evaluate the contribution of each component in MLGIB, by fixing the instantiation of the lower bound and replacing the instantiation of AIB and HIB with the original graph structure $A$ and node embeddings $Z_{X}^{(l)}$, respectively. Table \ref{tab:ablation_study} shows:
 (1) \textbf{Both components are essential.} 
    Removing either leads to consistent performance degradation, demonstrating that both informative message-passing paths and purified messages are necessary for effective multi-label message passing.
(2) \textbf{Effect of removing AIB Instantiation.} 
    Without it, MLGIB lacks the constraint on redundant message paths from source-node data, increasing the difficulty of message purification when there are many redundant paths.
(3) \textbf{Effect of removing HIB Instantiation.} 
    Without it, performance drops significantly, especially on the Delve dataset, further strengthening our insight: \textbf{In multi-label graphs, over-squashing is not merely caused by structural constraints, but by whether informative signals can be extracted from partially overlapping label information.} This demonstrates that we have substantially alleviated over-squashing in multi-label graphs.

\section{Conclusion}
%We propose a novel theoretical framework MLGIB to alleviate the over-squashing problem in multi-label graphs. Initially, we theoretically establish the theoretical foundation of MLGIB by constructing a Markov-chain search space and deriving the corresponding variational upper and lower bounds. Subsequently, we instantiate the MLGIB framework by incorporating the distinctive properties of multi-label graphs, resulting in a practical message-passing methodspecifically designed for this context. Finally, extensive experiments on four real-world multi-label datasets demonstrate our approach's superiority over baseline and SOTA models.

In this paper, we investigated over-squashing in multi-label graphs, where message passing not only compresses long-range information but also mixes predictive label signals with irrelevant noise caused by partial label overlap. To address this issue, we proposed the MLGIB, which formulates multi-label message passing as constrained information transmission. By balancing expressiveness and robustness, MLGIB preserves predictive label signals while suppressing irrelevant noise through tractable variational bounds and an end-to-end differentiable architecture. Experiments on real-world multi-label datasets show that MLGIB consistently outperforms representative GNNs, graph rewiring methods, graph information bottleneck methods, and state-of-the-art multi-label graph models. 

%Further analyses demonstrate its stronger label-wise expressiveness, generalization ability, and robustness under weak label correlation, high node degree, and structural perturbations.
\bibliography{references}

@inproceedings{DBLP:conf/iclr/0002Y21,
  author       = {Uri Alon and
                  Eran Yahav},
  title        = {On the Bottleneck of Graph Neural Networks and its Practical Implications},
  booktitle    = {{ICLR} 2021,
                  Virtual Event, Austria, May 3-7, 2021},
  publisher    = {OpenReview.net},
  year         = {2021},
  bibsource    = {dblp computer science bibliography, https://dblp.org}
}

@inproceedings{DBLP:conf/iclr/ToppingGC0B22,
  author       = {Jake Topping and
                  Francesco Di Giovanni and
                  Benjamin Paul Chamberlain and
                  Xiaowen Dong and
                  Michael M. Bronstein},
  title        = {Understanding over-squashing and bottlenecks on graphs via curvature},
  booktitle    = {{ICLR}
                  2022, Virtual Event, April 25-29, 2022},
  publisher    = {OpenReview.net},
  year         = {2022},
  bibsource    = {dblp computer science bibliography, https://dblp.org}
}

@inproceedings{DBLP:conf/iclr/KipfW17,
  author       = {Thomas N. Kipf and
                  Max Welling},
  title        = {Semi-Supervised Classification with Graph Convolutional Networks},
  booktitle    = {5th International Conference on Learning Representations, {ICLR} 2017,
                  Toulon, France, April 24-26, 2017, Conference Track Proceedings},
  publisher    = {OpenReview.net},
  year         = {2017},
  url          = {https://openreview.net/forum?id=SJU4ayYgl},
  bibsource    = {dblp computer science bibliography, https://dblp.org}
}

@article{sen2008collective,
  title={Collective classification in network data},
  author={Sen, Prithviraj and Namata, Galileo and Bilgic, Mustafa and Getoor, Lise and Galligher, Brian and Eliassi-Rad, Tina},
  journal={AI magazine},
  volume={29},
  number={3},
  pages={93--93},
  year={2008}
}

@inproceedings{nguyen2023revisiting,
  title={Revisiting over-smoothing and over-squashing using ollivier-ricci curvature},
  author={Nguyen, Khang and Hieu, Nong Minh and Nguyen, Vinh Duc and Ho, Nhat and Osher, Stanley and Nguyen, Tan Minh},
  booktitle={International Conference on Machine Learning},
  pages={25956--25979},
  year={2023},
  organization={PMLR}
}

@inproceedings{karhadkar2023fosr,
  title={FoSR: First-order spectral rewiring for addressing oversquashing in GNNs},
  author={Karhadkar, Kedar and Banerjee, Pradeep and Montufar, Guido},
  booktitle={International Conference on Learning Representations},
  year={2023}
}

@article{xiao2022semantic,
  title={Semantic guide for semi-supervised few-shot multi-label node classification},
  author={Xiao, Lin and Xu, Pengyu and Jing, Liping and Akujuobi, Uchenna and Zhang, Xiangliang},
  journal={Information Sciences},
  volume={591},
  pages={235--250},
  year={2022},
  publisher={Elsevier}
}

@inproceedings{gao2019semi,
  title={Semi-supervised graph embedding for multi-label graph node classification},
  author={Gao, Kaisheng and Zhang, Jing and Zhou, Cangqi},
  booktitle={International Conference on Web Information Systems Engineering},
  pages={555--567},
  year={2019},
  organization={Springer}
}

@article{zhou2021multi,
  title={Multi-label graph node classification with label attentive neighborhood convolution},
  author={Zhou, Cangqi and Chen, Hui and Zhang, Jing and Li, Qianmu and Hu, Dianming and Sheng, Victor S},
  journal={Expert Systems with Applications},
  volume={180},
  pages={115063},
  year={2021},
  publisher={Elsevier}
}

@article{wu2020graph,
  title={Graph information bottleneck},
  author={Wu, Tailin and Ren, Hongyu and Li, Pan and Leskovec, Jure},
  journal={Advances in Neural Information Processing Systems},
  volume={33},
  pages={20437--20448},
  year={2020}
}

@inproceedings{alemi2017deep,
  title={Deep Variational Information Bottleneck},
  author={Alemi, Alexander A and Fischer, Ian and Dillon, Joshua V and Murphy, Kevin},
  booktitle={International Conference on Learning Representations},
  year={2017}
}

@article{mikolov2013distributed,
  title={Distributed representations of words and phrases and their compositionality},
  author={Mikolov, Tomas and Sutskever, Ilya and Chen, Kai and Corrado, Greg S and Dean, Jeff},
  journal={Advances in neural information processing systems},
  volume={26},
  year={2013}
}

@article{nguyen2010estimating,
  title={Estimating divergence functionals and the likelihood ratio by convex risk minimization},
  author={Nguyen, XuanLong and Wainwright, Martin J and Jordan, Michael I},
  journal={IEEE Transactions on Information Theory},
  volume={56},
  number={11},
  pages={5847--5861},
  year={2010},
  publisher={IEEE}
}

@inproceedings{bei2025correlation,
  title={Correlation-aware graph convolutional networks for multi-label node classification},
  author={Bei, Yuanchen and Chen, Weizhi and Chen, Hao and Zhou, Sheng and Yang, Carl and Fan, Jiapei and Huang, Longtao and Bu, Jiajun},
  booktitle={Proceedings of the 31st ACM SIGKDD Conference on Knowledge Discovery and Data Mining V. 1},
  pages={37--48},
  year={2025}
}

@inproceedings{akujuobi2019collaborative,
  title={Collaborative graph walk for semi-supervised multi-label node classification},
  author={Akujuobi, Uchenna and Yufei, Han and Zhang, Qiannan and Zhang, Xiangliang},
  booktitle={2019 IEEE International Conference on Data Mining (ICDM)},
  pages={1--10},
  year={2019},
  organization={IEEE}
}

@article{lakshmanan2010knowledge,
  title={Knowledge discovery in the blogosphere: Approaches and challenges},
  author={Lakshmanan, Geetika and Oberhofer, Martin},
  journal={IEEE internet computing},
  volume={14},
  number={2},
  pages={24--32},
  year={2010},
  publisher={IEEE}
}

@article{zhao2023multi,
  title={Multi-label node classification on graph-structured data},
  author={Zhao, Tianqi and Dong, Ngan Thi and Hanjalic, Alan and Khosla, Megha},
  journal={arXiv preprint arXiv:2304.10398},
  year={2023}
}

@inproceedings{velivckovic2018graph,
  title={Graph Attention Networks},
  author={Veli{\v{c}}kovi{\'c}, Petar and Cucurull, Guillem and Casanova, Arantxa and Romero, Adriana and Li{\`o}, Pietro and Bengio, Yoshua},
  booktitle={International Conference on Learning Representations},
  year={2018}
}

@book{cover2012elements,
  title={Elements of information theory},
  author={Cover, Thomas M. and Thomas, Joy A.},
  year={2012},
  publisher={John Wiley \& Sons}
}

@article{jaccard1912distribution,
  title={The distribution of the flora in the alpine zone. 1},
  author={Jaccard, Paul},
  journal={New phytologist},
  volume={11},
  number={2},
  pages={37--50},
  year={1912},
  publisher={Wiley Online Library}
}

@inproceedings{DBLP:conf/nips/JamadandiRB24,
  author       = {Adarsh Jamadandi and
                  Celia Rubio{-}Madrigal and
                  Rebekka Burkholz},
  title        = {Spectral Graph Pruning Against Over-Squashing and Over-Smoothing},
  booktitle    = {NeurIPS},
  year         = {2024}
}

@inproceedings{DBLP:conf/nips/ZhuYZHAK20,
  author       = {Jiong Zhu and
                  Yujun Yan and
                  Lingxiao Zhao and
                  Mark Heimann and
                  Leman Akoglu and
                  Danai Koutra},
  title        = {Beyond Homophily in Graph Neural Networks: Current Limitations and
                  Effective Designs},
  booktitle    = {NeurIPS},
  year         = {2020}
}

@inproceedings{yu2014large,
  title={Large-scale multi-label learning with missing labels},
  author={Yu, Hsiang-Fu and Jain, Prateek and Kar, Purushottam and Dhillon, Inderjit},
  booktitle={International conference on machine learning},
  pages={593--601},
  year={2014},
  organization={PMLR}
}

@article{bhatia2015sparse,
  title={Sparse local embeddings for extreme multi-label classification},
  author={Bhatia, Kush and Jain, Himanshu and Kar, Purushottam and Varma, Manik and Jain, Prateek},
  journal={Advances in neural information processing systems},
  volume={28},
  year={2015}
}

@article{tishby2000information,
  title={The information bottleneck method},
  author={Tishby, Naftali and Pereira, Fernando C and Bialek, William},
  journal={arXiv preprint physics/0004057},
  year={2000}
}

@inproceedings{tishby2015deep,
  title={Deep learning and the information bottleneck principle},
  author={Tishby, Naftali and Zaslavsky, Noga},
  booktitle={2015 ieee information theory workshop (itw)},
  pages={1--5},
  year={2015},
  organization={Ieee}
}

@article{saxe2019information,
  title={On the information bottleneck theory of deep learning},
  author={Saxe, Andrew M and Bansal, Yamini and Dapello, Joel and Advani, Madhu and Kolchinsky, Artemy and Tracey, Brendan D and Cox, David D},
  journal={Journal of Statistical Mechanics: Theory and Experiment},
  volume={2019},
  number={12},
  pages={124020},
  year={2019},
  publisher={IOP Publishing and SISSA}
}

@inproceedings{DBLP:conf/icml/LiangBSXPS25,
  author       = {Langzhang Liang and
                  Fanchen Bu and
                  Zixing Song and
                  Zenglin Xu and
                  Shirui Pan and
                  Kijung Shin},
  title        = {Mitigating Over-Squashing in Graph Neural Networks by Spectrum-Preserving
                  Sparsification},
  booktitle    = {{ICML}},
  series       = {Proceedings of Machine Learning Research},
  publisher    = {{PMLR} / OpenReview.net},
  year         = {2025}
}

@inproceedings{DBLP:conf/cikm/SaberS25,
  author       = {Danial Saber and
                  Amirali Salehi{-}Abari},
  title        = {Empirical Study of Over-Squashing in GNNs and Causal Estimation of
                  Rewiring Strategies},
  booktitle    = {{CIKM}},
  pages        = {2525--2534},
  publisher    = {{ACM}},
  year         = {2025}
}

@inproceedings{DBLP:conf/aaai/AziziKHB26,
  author       = {Niloofar Azizi and
                  Nils M. Kriege and
                  Nicholas J. A. Harvey and
                  Horst Bischof},
  title        = {Spectral Basis Learning for Expressive Graph Neural Networks in Link
                  Prediction},
  booktitle    = {{AAAI}},
  pages        = {19640--19648},
  publisher    = {{AAAI} Press},
  year         = {2026}
}

@inproceedings{DBLP:conf/aaai/HevapathigeWZ26,
  author       = {Asela Hevapathige and
                  Asiri Wijesinghe and
                  Ahad N. Zehmakan},
  title        = {Beyond Fixed Depth: Adaptive Graph Neural Networks for Node Classification
                  Under Varying Homophily},
  booktitle    = {{AAAI}},
  pages        = {21717--21725},
  publisher    = {{AAAI} Press},
  year         = {2026}
}

@inproceedings{DBLP:conf/aaai/LiuY26,
  author       = {Lihui Liu and
                  Yuchen Yan},
  title        = {{MORGAN:} To Bridge Mixture of Experts and Spectral Graph Neural Network},
  booktitle    = {{AAAI}},
  pages        = {23783--23791},
  publisher    = {{AAAI} Press},
  year         = {2026}
}

@inproceedings{DBLP:conf/iclr/BergnaCOLH25,
  author       = {Richard Bergna and
                  Sergio Calvo{-}Ordo{\~{n}}ez and
                  Felix L. Opolka and
                  Pietro Lio and
                  Jos{\'{e}} Miguel Hern{\'{a}}ndez{-}Lobato},
  title        = {Uncertainty Modeling in Graph Neural Networks via Stochastic Differential
                  Equations},
  booktitle    = {{ICLR}},
  publisher    = {OpenReview.net},
  year         = {2025}
}

@inproceedings{DBLP:conf/iclr/LiG0025,
  author       = {Shouheng Li and
                  Floris Geerts and
                  Dongwoo Kim and
                  Qing Wang},
  title        = {Towards Bridging Generalization and Expressivity of Graph Neural Networks},
  booktitle    = {{ICLR}},
  publisher    = {OpenReview.net},
  year         = {2025}
}

@article{wu2025graph,
  title={Graph contrastive learning on multi-label classification for recommendations},
  author={Wu, Jiayang and Gan, Wensheng and Lu, Huashen and Yu, Philip S},
  journal={ACM Transactions on Intelligent Systems and Technology},
  volume={16},
  number={4},
  pages={1--19},
  year={2025},
  publisher={ACM New York, NY}
}

@inproceedings{sun2025multi,
  title={Multi-label node classification with label influence propagation},
  author={Sun, Yifei and Liu, Zemin and Hooi, Bryan and Yang, Yang and Fathony, Rizal and Chen, Jia and He, Bingsheng},
  booktitle={The Thirteenth International Conference on Learning Representations},
  year={2025}
}

@inproceedings{wo2025local,
  title={Local Homophily-Aware Graph Neural Network with Adaptive Polynomial Filters for Scalable Graph Anomaly Detection},
  author={Wo, Zengyi and Shao, Minglai and Zhang, Shiyu and Wang, Ruijie},
  booktitle={Proceedings of the 31st ACM SIGKDD Conference on Knowledge Discovery and Data Mining V. 2},
  pages={3180--3191},
  year={2025}
}

@inproceedings{loveland2025unveiling,
  title={Unveiling the impact of local homophily on gnn fairness: In-depth analysis and new benchmarks},
  author={Loveland, Donald and Koutra, Danai},
  booktitle={Proceedings of the 2025 SIAM International Conference on Data Mining (SDM)},
  pages={608--617},
  year={2025},
  organization={SIAM}
}

@inproceedings{jiang2026does,
  title={Does Homophily Help in Robust Test-time Node Classification?},
  author={Jiang, Yan and Qiu, Ruihong and Huang, Zi},
  booktitle={Proceedings of the Nineteenth ACM International Conference on Web Search and Data Mining},
  pages={271--281},
  year={2026}
}

@inproceedings{chen2025beyond,
  title={Beyond Homophily: Graph Contrastive Learning with Macro-Micro Message Passing},
  author={Chen, Yiyuan and Guan, Donghai and Yuan, Weiwei and Zang, Tianzi},
  booktitle={Proceedings of the AAAI Conference on Artificial Intelligence},
  volume={39(15)},
  pages={15948--15956},
  year={2025}
}

@inproceedings{walkega2025expressive,
  title={Expressive power of temporal message passing},
  author={Wa{\l}{\k{e}}ga, Przemys{\l}aw Andrzej and Rawson, Michael},
  booktitle={Proceedings of the AAAI Conference on Artificial Intelligence},
  volume={39(20)},
  pages={21000--21008},
  year={2025}
}

@inproceedings{kingma2014auto,
  title={Auto-Encoding Variational Bayes},
  author={Kingma, Diederik P. and Welling, Max},
  booktitle={International Conference on Learning Representations (ICLR)},
  year={2014}
}
%\bibliographystyle{plain}
% \newpage %supple

%%%%%%%%%%%%%%%%%%%%%%%%%%%%%%%%%%%%%%%%%%%%%%%%%%%%%%%%%%%%
\appendix
\section{Theoretical Supplement}
\label{proof_supplement}
\subsection{Proof of Proposition \ref{proposition:lower_bound}}
\begin{proof}
To derive the lower bound of the mutual information $I(Y_t; Z_H^{(L)})$, we utilize the Nguyen-Wainwright-Jordan lower bound \citep{nguyen2010estimating}. For any two distributions $p(X)$ and $q(X)$:
\begin{equation}
    \text{KL}(p \| q) = \sup_{g} \left( \mathbb{E}_{p(X)}[g(X)] - \mathbb{E}_{q(X)}[\text{exp}({g(X)-1})] \right).
\end{equation}
Recall that $I(Y_t; Z_H^{(L)}) = \text{KL} \big( \mathbb{P}(Y_t, Z_H^{(L)}) \parallel \mathbb{P}(Y_t)\mathbb{P}(Z_H^{(L)}) \big)$, for any measurable function $g(Y_t, Z_H^{(L)})$, we have:
\begin{equation}
    I(Y_t; Z_H^{(L)}) \geq \mathbb{E}_{\mathbb{P}(Y_t, Z_H^{(L)})} [g(Y_t, Z_H^{(L)})] 
    - \mathbb{E}_{\mathbb{P}(Y_t)\mathbb{P}(Z_H^{(L)})} [\text{exp}({g(Y_t, Z_H^{(L)})-1})].
\end{equation}
Define the variational function $g(Y_t, Z_H^{(L)}) = 1 + \log \frac{\mathbb{Q}_1(Y_t \mid Z_H^{(L)})}{\mathbb{Q}_2(Y_t)}$. Substituting this into the inequality, the first term becomes:
\begin{equation}
    \mathbb{E}_{\mathbb{P}(Y_t, Z_H^{(L)})} \left[ 1 + \log \frac{\mathbb{Q}_1(Y_t \mid Z_H^{(L)})}{\mathbb{Q}_2(Y_t)} \right] = 1 + \mathbb{E}_{\mathbb{P}(Y_t, Z_H^{(L)})} \left[ \log \frac{\mathbb{Q}_1(Y_t \mid Z_H^{(L)})}{\mathbb{Q}_2(Y_t)} \right].
\end{equation}
The second term simplifies as:
\begin{equation}
    \mathbb{E}_{\mathbb{P}(Y_t)\mathbb{P}(Z_H^{(L)})} \left[ \text{exp}({(1 + \log \frac{\mathbb{Q}_1(Y_t \mid Z_H^{(L)})}{\mathbb{Q}_2(Y_t)}) - 1}) \right] = \mathbb{E}_{\mathbb{P}(Y_t)\mathbb{P}(Z_H^{(L)})} \left[ \frac{\mathbb{Q}_1(Y_t \mid Z_H^{(L)})}{\mathbb{Q}_2(Y_t)} \right].
\end{equation}
Combining these, we arrive at the lower bound:
\begin{equation}
    I(Y_t; Z_H^{(L)}) \geq 1 + \mathbb{E}_{\mathbb{P}(Y_t, Z_H^{(L)})} \left[ \log \frac{\mathbb{Q}_1(Y_t \mid Z_H^{(L)})}{\mathbb{Q}_2(Y_t)} \right] 
    - \mathbb{E}_{\mathbb{P}(Y_t)\mathbb{P}(Z_H^{(L)})} \left[ \frac{\mathbb{Q}_1(Y_t \mid Z_H^{(L)})}{\mathbb{Q}_2(Y_t)} \right],
\end{equation}
which concludes the proof.
\end{proof}

\subsection{Proof of Proposition \ref{proposition:upper_bound}}
\begin{proof}
Considering the source data $\mathcal{D}_s = (A, X)$ and the Markovian dependence of $\mathcal{M}$, we establish the Markov chain:
\begin{equation*}
    \mathcal{D}_s \to \{Z_H^{(l)}\}_{l \in S_H} \cup \{Z_A^{(l)}\}_{l \in S_A} \to Z_H^{(L)}.
\end{equation*}
By the Data Processing Inequality \citep{cover2012elements}, the upper bound of the target term is:
\begin{equation}
    I(\mathcal{D}_s; Z_H^{(L)}) \leq I\left(\mathcal{D}_s; \{Z_H^{(l)}\}_{l \in S_H} \cup \{Z_A^{(l)}\}_{l \in S_A}\right).
\end{equation}
To decompose the above term, we introduce the following history sets:
\begin{align}
    H_A^{(l)} &= \{Z_H^{(l_1)}, Z_A^{(l_2)} \mid l_1 < l, l_2 < l, l_1 \in S_H, l_2 \in S_A\} ,\nonumber \\
    H_H^{(l)} &= \{Z_H^{(l_1)}, Z_A^{(l_2)} \mid l_1 \leq l, l_2 < l, l_1 \in S_H, l_2 \in S_A\}. \nonumber
\end{align}
Applying the Chain Rule for Mutual Information \citep{cover2012elements}, the above term is expanded as:
\begin{equation}
    I\big(\mathcal{D}_s; \{Z_H^{(l)}\}_{l \in S_H} \cup \{Z_A^{(l)}\}_{l \in S_A}\big) 
    = \sum_{l \in S_A} I(\mathcal{D}_s; Z_A^{(l)} | H_A^{(l)}) + \sum_{l \in S_H} I(\mathcal{D}_s; Z_H^{(l)} | H_H^{(l)}).
\end{equation}
For each term $l \in S_A$, by the Chain Rule for Mutual Information, and considering the Markovian property of $(A, Z_X^{(l)}) \to Z_A^{(l)}$, we have:
{\small
\begin{align}
    I(\mathcal{D}_s; Z_A^{(l)} | H_A^{(l)}) &\leq I(\mathcal{D}_s, A, Z_X^{(l)}; Z_A^{(l)} | H_A^{(l)}) \nonumber \\
    &= I(A, Z_X^{(l)}; Z_A^{(l)} | H_A^{(l)}) + I(\mathcal{D}_s; Z_A^{(l)} | A, Z_X^{(l)}, H_A^{(l)}) \nonumber \\
    &\leq I(Z_A^{(l)}; A, Z_X^{(l)}).
\end{align}
}

Similarly, for $l \in S_H$, given the Markovian property of $(Z_A^{(l)}, Z_X^{(l)}) \to Z_H^{(l)}$, we obtain the following inequality:
\begin{equation}
I(\mathcal{D}_s; Z_H^{(l)} | H_H^{(l)}) \leq I(Z_H^{(l)}; Z_A^{(l)}, Z_X^{(l)}).
\end{equation}
To provide a computationally tractable objective, we introduce a variational approximation to the marginal distribution $p(y)$. Note that the mutual information can be expressed as:
\begin{equation*}
    I(X; Y) = \mathbb{E}_{p(x,y)} \left[ \log \frac{p(y|x)}{p(y)} \right].
\end{equation*}
Given an arbitrary variational distribution $q(y)$, we have:
\begin{equation*}
    \mathbb{E}_{p(x,y)} \left[ \log \frac{p(y|x)}{q(y)} \right] - I(X; Y) = \text{KL}(p(y) \| q(y)) \geq 0,
\end{equation*}
which leads to the variational upper bound: 
\begin{equation}
    I(X; Y) \leq \mathbb{E}_{p(x,y)} [ \log \frac{p(y|x)}{q(y)} ].
\end{equation}
By applying this property to each term in the summation and substituting the definitions from Eq. \ref{eq:upper_bound}, we arrive at the final inequality:
\begin{equation}
    I(\mathcal{D}_s; Z_H^{(L)}) \leq \sum_{l \in S_A} \text{AIB}^{(l)} + \sum_{l \in S_H} \text{HIB}^{(l)},
\end{equation}
which concludes the proof.
\end{proof}

\section{Computational Complexity}
We analyze the complexity of the instantiated MLGIB under an efficient implementation based on sparse message passing and stochastic neighborhood sampling. Let $n$ and $m$ denote the numbers of nodes and edges, respectively; let $l$ be the number of propagation layers; let $h$ be the hidden dimensionality; let $C$ be the number of labels; let $d$ be the label embedding dimensionality; and let $N_b$ and $E_b$ denote the numbers of sampled nodes and edges in a mini-batch subgraph.

For full-graph training, the dominant cost of one MLGIB layer consists of three parts: 
(i) label-aware neighbor scoring in AIB, which requires $O(nCd + mC)$ operations to construct label-aware node representations and compute edge-wise relevance scores; 
(ii) hierarchical information bottleneck modeling in HIB, which incurs $O(mh^2)$ time for messages computation; and 
(iii) standard  message aggregation and node update, which can be implemented in $O(mh + nh^2)$ or $O(mh^2)$. Assuming $m > n$, we use an efficient implementation with a complexity of $O(mh + nh^2)$.
Therefore, the overall time complexity of a $l$-layer MLGIB is
    $O\!\left(
l(nCd + mC + mh^2  + nh^2)
\right)$. In sparse graphs, this can be viewed as near-linear in the number of edges up to moderate feature- and label-dependent factors.
In practice, MLGIB is trained with stochastic neighborhood sampling rather than full-graph propagation. Under this setting, the per-mini-batch complexity becomes
\begin{equation}
    O\!\left(
l(N_bCd + E_bC + E_bh^2  + N_bh^2)
\right),
\end{equation}
which scales with the sampled subgraph size instead of the entire graph. This significantly improves practicality on large-scale graphs and makes the method compatible with standard mini-batch GNN training pipelines.

The space complexity mainly includes node representations $O(nh)$, pseudo-label embeddings $O(nC)$, and message embeddings $O(mh)$. Full-graph storage requires $O(nh + mh + nC)$ memory, while the mini-batch version only stores sampled node/edge states, yielding
\begin{equation}
    O(N_bh + E_bh + N_bC).
\end{equation}

Overall, MLGIB preserves the sparse-computation advantages of modern GNN frameworks, while introducing only structured overheads associated with label-aware filtering and information-theoretic message compression. This design leads to a scalable implementation suitable for both medium- and large-scale multi-label graph learning.

\section{Experiment Setup Supplement}
\label{Experiment_supplement}
\subsection{Dataset Supplement}
\label{Dataset_supplement}
To evaluate the effectiveness of the proposed framework, we conduct experiments on four real-world multi-label datasets that exhibit diverse characteristics:
\begin{itemize}
    \item \textbf{DBLP} \citep{akujuobi2019collaborative} depicts the co-authorship relation between authors. Each node represents an author, and each edge represents the existence of a co-authorship relation between two authors. The dataset comprises four distinct labels, each corresponding to a specific research domain. The feature vector for each author (node) is constructed by concatenating the textual content of the titles from all publications authored by the individual.
    \item \textbf{BlogCatalog} \citep{lakshmanan2010knowledge} is a prevalent benchmark in blogosphere analysis and network representation learning. In this network, nodes correspond to individual bloggers on the BlogCatalog platform. An edge is established between two nodes to represent a mutual friendship between the corresponding bloggers. Node labels are derived from the thematic categories of the blogs published by each user. Notably, this dataset does not provide intrinsic node attributes. To evaluate the efficacy of our proposed methodology and established baselines on node classification tasks within attribute-deficient networks, we employ randomly initialized node embeddings for the experimental validation.
    \item \textbf{PCG} \citep{zhao2023multi} is employed for protein phenotype prediction. It consists of a network containing 3,233 nodes (proteins) and 37,351 edges. Each node is represented by a 32-dimensional feature vector. The multi-label classification task involves assigning nodes to one or more of 15 phenotype categories, where a phenotype denotes an observable characteristic associated with a disease. Edges in the graph represent functional interactions between protein pairs. The correspondence linking proteins to their associated phenotypes is curated from an open-source database.
    % \item \textbf{Humloc} \citep{zhao2023multi} is designed for the prediction of human protein subcellular localization. It comprises a network of 3,106 nodes and 18,496 edges. Each node, representing a protein, is associated with one or more labels from a set of 14 possible subcellular locations. The graph structure is constructed based on protein-protein interaction data sourced from a publicly available database, where an edge is established between two protein nodes if a known interaction exists between the corresponding biological entities.
    \item \textbf{Delve} \citep{xiao2022semantic} constitutes a large-scale citation network comprising 1,229,280 nodes (academic papers) and 4,322,275 edges. Within this graph, an edge is established from one node to another to represent a citation relationship between the corresponding publications. The classification task involves assigning each paper a subset of 20 predefined research field categories (indexed from 1 to 20), which serve as node labels.
\end{itemize}
A summary of the dataset statistics is reported in Table \ref{tab:static}. We also analyze the distributions of label correlation among nodes, as shown in Fig. \ref{fig:statistic}. Except for DBLP, most nodes in the other datasets exhibit low label similarity with their neighbors, which also validates the characteristic of partial label overlap for nodes in multi-label graphs that we proposed in Fig. \ref{fig:intro}(a).
\begin{table}[htbp]
  \centering
  \caption{Dataset Statistics}
    \begin{tabular}{ccccc}
    \toprule
    Dataset & Nodes & Features & Classes & Edges \\
    \midrule
    DBLP  & 28,702 & 300   & 4     & 68,335 \\
    BlogCatalog  & 10,312 & 300   & 39    & 333,983 \\
    PCG   & 3,233  & 32    & 15    & 37,351 \\
    % Humloc & 3,106  & 32    & 14    & 18,496 \\
    Delve & 1,229,280 & 300   & 20    & 4,322,275 \\
    \bottomrule
    \end{tabular}%
  \label{tab:static}%
\end{table}%

\begin{figure}
    \centering
    \includegraphics[width=1\linewidth]{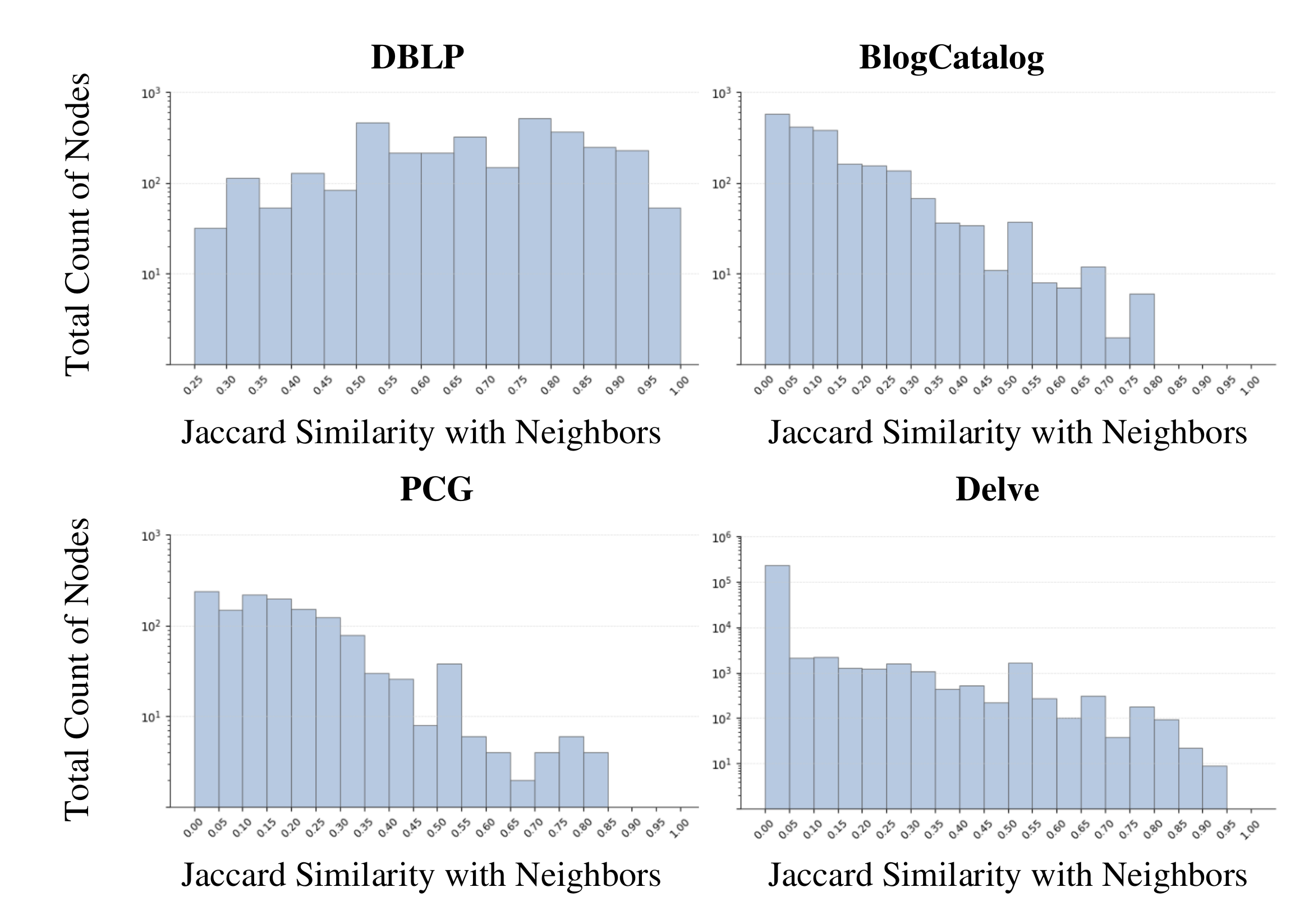}
    \caption{Distributions of label correlation among nodes.}
    \label{fig:statistic}
\end{figure}

\textbf{PubMed dataset (single-label) \citep{sen2008collective} in expressiveness analysis. } PubMed dataset is the standard node classification benchmark used in graph neural networks (GNNs). It models the citation network of biomedical literature. In this dataset, nodes represent scientific publications, and edges represent citation links. The dataset contains 19,717 nodes and 44,338 edges.

DBLP and BlogCatalog datasets are available at: https://github.com/Tianqi-py/MLGNC; PCG and Delve datasets are available at: https://github.com/YuanchenBei/CorGCN; The PubMed dataset is available via the torch\_geometric function in PyTorch. These sources are not official releases. We randomly split the data into 60\% for training, 20\% for validation, and 20\% for testing. All results are averaged over 10 independent runs with different random seeds. 

\subsection{Metric Supplement}
\label{Metric_supplement}
Following prior work on multi-label graph models \citep{gao2019semi,zhou2021multi,bei2025correlation}, to comprehensively evaluate the performance of our multi-label graph model, we employed a suite of seven established metrics: Macro-AUC, Micro-AUC, Ranking Loss, Hamming Loss, Macro-AP, Micro-AP, and LRAP. 
\begin{itemize}
    \item Macro-AUC calculates the Area Under the Receiver Operating Characteristic Curve for each label independently and then averages the results, providing a class-aware measure that treats all labels equally. It is sensitive to the performance on minority labels.
    \item Micro-AUC computes the AUC by globally aggregating all predictions and true labels across all instances and labels, effectively treating the evaluation as a single binary classification task. It is more influenced by the performance on majority labels.
    \item Ranking Loss measures the average fraction of reversely ordered label pairs, where an irrelevant label is incorrectly ranked above a relevant one. It assesses the quality of the predictive ranking.
    \item Hamming Loss quantifies the fraction of misclassified instance-label pairs (i.e., wrong labels) to the total number of labels. It is a straightforward measure of overall classification error.
    \item Macro-AP computes the average precision (area under the precision-recall curve) for each label and then takes the arithmetic mean over all labels, offering a label-centric perspective on precision-recall trade-off.
    \item Micro-AP  calculates the average precision from a globally aggregated contingency table of all predictions and true labels, providing an instance-centric perspective on overall precision and recall.
    \item LRAP (Label Ranking Average Precision) evaluates the quality of the ranked label list for each instance. For a given instance, it calculates the average precision over the set of its relevant labels, based on their ranks within the predicted ordering, and then averages this value across all instances. It directly measures how well the ranking surface relevant labels at the top.
\end{itemize}
\subsection{Compute Resources}
\label{Compute_Resource}
% All experiments were conducted on a single A6000 GPU with 50 GB memory. It should be noted, however, that some baseline methods—due to a lack of proper optimization—require far more memory than this configuration provides.
All experiments were conducted on a server running Ubuntu 22.04 LTS, equipped with an Intel Core i9-10920X CPU (3.50 GHz base, 4.80 GHz max turbo frequency, 12 cores / 24 threads), 125 GB of system memory, and an NVIDIA RTX A6000 50GB GPU. The software environment included Python 3.9.18, PyTorch 2.4.0, and CUDA 12.4.
\subsection{Training Configuration}
Our training process implements a graph information bottleneck-inspired approach for multi-label node classification. The configuration details are as follows:
\begin{itemize}
    \item \textbf{Optimizer:} Adam with a learning rate of $1 \times 10^{-3}$.
    % \item \textbf{Loss Function:} Multi-label Binary Cross Entropy with auxiliary Information Bottleneck regularizers
    \item \textbf{Neighborhood Sampling:} Stochastic neighbor sampling with a fixed fanout of 1,024 per node at each layer.
    \item \textbf{Regularization:} The loss is employed with a dataset-specific hyperparameter $\beta \in [0, 10^{-3}]$.
    \item \textbf{Gaussian Mixture kernels:} For the Gaussian mixture distribution defined in the HIB module, we fix the number of mixture kernels to $m=6$.
    \item \textbf{Reproducibility:} Deterministic execution using fixed random seeds and \texttt{CUBLAS\_\allowbreak WORKSPACE\_\allowbreak CONFIG} set to \texttt{:4096:8}.
    \item \textbf{Initialization:} Model parameters are initialized using Xavier uniform initialization to maintain stable gradient flow.
    \item \textbf{Mini-batch Training:} To handle the scale of the \texttt{Delve} dataset, we employ stochastic neighbor sampling with a fixed fanout of 1,024 for each target node per layer.
    \item \textbf{Embedding Dimension:} The dimension of the label embedding space is set to $d = 128$.
    \item \textbf{Skip-gram Negative Sampling:} The positive-to-negative sample ratio is 1:5.
    %For each positive label pair within a node, we perform strict negative sampling with $K = 5$ to contrast co-occurring labels against non-existent ones
    \item \textbf{Skip-gram Batch Configuration:} To ensure efficient gradient updates on large-scale datasets like \texttt{Delve}, we utilize a batch size of 1,024 for the label pairs.
\end{itemize}
% The source codes are released at: https://anonymous.4open.science/r/MLGIB-13776/.

\section{Related Works and Discussion}
In this section, we review the most related works, discussing the differences and connections between our work and existing methods. 
\subsection{Multi-Label Graph Learning}
As previously discussed, only a few studies have focused on Multi-Label Graph Learning. Representatively, ML-GCN \citep{gao2019semi} employs a skip‑gram training strategy to learn the three‑way relations between labels and labels, as well as between labels and nodes. Through joint optimization, it jointly derives the embeddings for both labels and nodes. This provides inspiration for the weight assignment of messages in our theoretical instantiation. LANC \citep{zhou2021multi} incorporates a label‑attention mechanism at the final layer, which inserts a label vector correlated with the features. This vector comprises a combination of the label information most relevant to the features, and is used to predict the final labels. Similar to the idea of assigning message‑passing weights according to label relevance in this study, LARN \citep{xiao2022semantic} proposes a Label‑Guided Neighbor Aggregator Module to aggregate node neighbors with different weights. Going further, CorGCN \citep{bei2025correlation} learns separate graphs for the number of labels in order to disentangle ambiguous features and structure. Information propagation and aggregation are performed on each graph, and the results are then merged for final prediction.

\textbf{Discussion.} Although the aforementioned methods leverage attention mechanisms or multi‑label correlations to achieve a certain degree of effectiveness, they lack a principled information‑theoretic foundation. Consequently, they fail to filter and purify multi‑label message propagation, and thus remain limited by the multi‑label over‑squashing problem caused by the compression of diverse label information. Our proposed MLGIB method introduces the information bottleneck principle into multi-label graph message passing, filtering and purifying the information flow, thereby alleviating multi-label over-squashing by filtering noisy information during propagation.  

\subsection{Over-Squashing in Multi-Label Graphs}
As previously discussed, the over-squashing issue becomes prominent when GNNs attempt to convolve or aggregate information from long-range nodes in a graph~\citep{DBLP:conf/iclr/0002Y21}. As the number of GNN layers increases to reach distant nodes, information from an exponentially growing receptive field is forced to be compressed into fixed-length node embeddings. Thus, conventional GNNs struggle to effectively propagate messages from distant nodes and tend to overemphasize local information.

From a Riemannian geometry perspective, over-squashing has been attributed to negatively curved edges in the graph~\citep{DBLP:conf/iclr/ToppingGC0B22}. To address this issue, the Stochastic Discrete Ricci Flow (SDRF) algorithm modifies the graph structure by adding or removing edges around negatively curved regions. Similarly, BORF~\citep{nguyen2023revisiting} further establishes the connection between over-squashing and graph geometry, proposing a spatial graph rewiring strategy to alleviate this phenomenon. More recently, this line of research has linked over-squashing to graph curvature and spectral gaps, motivating the development of spectral rewiring methods such as FOSR~\citep{karhadkar2023fosr}, which introduces a computationally efficient approach to mitigate over-squashing by directly optimizing the spectral gap of the input graph.

\textbf{Discussion. }MLGIB differs from existing literature in three aspects. (1) \textbf{Multi-Label Awareness.} Existing over-squashing alleviation methods are primarily designed for single-label graphs and do not consider that over‑squashing is aggravated by the crowding of multiple label signals in multi‑label graphs.  
In contrast, MLGIB mitigates the over‑squashing problem on multi‑label graphs by assigning label‑relevant information to nodes. (2) \textbf{Information Preservation.} Curvature- and rewiring-based approaches alleviate over-squashing by modifying the input graph structure to reduce topological bottlenecks. However, such modifications may disrupt the original relational semantics of nodes, potentially leading to information loss in multi-label graphs. Because local homogeneity is diminished in multi-label graphs, altering the graph topology tends to inject additional irrelevant label information—essentially noise.
(3) \textbf{Targeted and Task-Adaptive Design.} Graph rewiring methods typically optimize structural properties such as curvature or spectral gaps in a task-agnostic manner. In contrast, MLGIB is formulated as an end-to-end framework. All instantiated components are designed to serve the objective function, giving the model a high degree of task‑driven focus, which in turn delivers higher performance for downstream multi‑label node‑classification tasks.

\section{Limitations}
\label{sec:limitations}
The core idea of MLGIB is to filter out label information irrelevant to the target node via the information bottleneck. Its effectiveness relies, to some extent, on the learnable label co‑occurrence patterns (i.e., label correlations) present in the graph. If the labels of a node and its neighbors are almost random and exhibit very weak correlation (i.e., the “homophily” is extremely low), the model may struggle to learn effective message‑passing paths ($Z_A^{(L)}$) and relevant messages ($Z_H^{(L)}$), potentially limiting performance gains.

\section{Broader Impact} 
\label{sec:broader Impact}
Although our work is primarily fundamental, focusing on enhancing the capabilities of GNNs to process graph-structured data, the improvements introduced by MLGIB can also significantly benefit various downstream tasks. These applications include a wide range of graph-related tasks, such as optimizing information flow in social networks, enhancing pattern recognition in fraud networks, and improving performance on other complex graph-based analyses.

\newpage

\end{document}